\documentclass[sigconf]{acmart}

\copyrightyear{2021}
\acmYear{2021} 
\setcopyright{iw3c2w3}
\acmConference[WWW '21]{Proceedings of the Web Conference 2021}{April 19--23, 2021}{Ljubljana, Slovenia}
\acmBooktitle{Proceedings of the Web Conference 2021 (WWW '21), April 19--23, 2021,
Ljubljana, Slovenia}

\acmPrice{}
\acmDOI{10.1145/3442381.3449965}
\acmISBN{978-1-4503-8312-7/21/04}

\AtBeginDocument{%
  \providecommand\BibTeX{{%
    \normalfont B\kern-0.5em{\scshape i\kern-0.25em b}\kern-0.8em\TeX}}}

\usepackage{xspace}
\usepackage{wrapfig}
\usepackage{macros/nicknames}
\usepackage{microtype}
\usepackage{graphicx}
\usepackage{subcaption}
\usepackage{booktabs} 
\usepackage{siunitx}
\usepackage[utf8]{inputenc}
\usepackage{xcolor}
\usepackage{amsmath} 
\usepackage{cancel}
\usepackage{algorithmic}
\usepackage{algorithm}

\usepackage{hyperref}

\theoremstyle{plain}
\newtheorem{thm}{Theorem}[section]

\theoremstyle{definition}

\newtheorem{lemma}[thm]{Lemma}

\usepackage[subtle]{savetrees}
\usepackage{multirow}
\usepackage{hyperref}
\setcopyright{iw3c2w3}

\begin{document}

\title{Not All Features Are Equal: Discovering  Essential Features for Preserving Prediction Privacy}


\author{Fatemehsadat Mireshghallah\quad Mohammadkazem Taram \quad  Ali Jalali$^\dagger$}
\author{Ahmed Taha Elthakeb$^\ast$\quad Dean Tullsen\quad Hadi Esmaeilzadeh}

\email{{fmireshg,mtaram}@eng.ucsd.edu,  ajjalali@amazon.com, ahmed.t.althakeb@gmail.com, {tullsen,hadi}@eng.ucsd.edu}

\affiliation{%
   \institution{University of California San Diego\\ $^\dagger$Amazon.com, Inc. \quad\quad\quad $^\ast$Samsung Electronics}
}

\renewcommand{\authors}{Fatemehsadat Mireshghallah, Mohammadkazem Taram,  Ali Jalali, Ahmed Taha Elthakeb, Dean Tullsen, and Hadi Esmaeilzadeh}

\renewcommand{\shortauthors}{Mireshghallah et al.}

\begin{abstract}
When receiving machine learning services from the cloud, the provider does not need to receive all features; in fact, only a subset of the features are necessary for the target prediction task.
Discerning this subset is the key problem of this work.
We formulate this problem as a gradient-based perturbation maximization method that discovers this subset in the input feature space with respect to the functionality of the prediction model used by the provider. 
After identifying the subset, our framework, \sieve, suppresses the rest of the features using utility-preserving constant values that are discovered through a separate gradient-based optimization process.
%
%
We show that \sieve does not necessarily require collaboration from the service provider beyond its normal service, and can be applied in scenarios where we only have black-box access to the service provider's model.
We theoretically guarantee that \sieve's optimizations reduce the upper bound of the Mutual Information (MI) between the data and the sifted representations that are sent out.
Experimental results show that \sieve reduces the mutual information between the input and the sifted representations by 85.01\%  with only negligible reduction in utility (1.42\%).
In addition, we show that \sieve greatly diminishes adversaries' ability to learn and infer non-conducive features. 

\end{abstract}



\maketitle

\vspace{-1ex}
\section{Introduction}
\label{sec:intro}
The computational complexity of Machine Learning (ML) models has pushed their execution to the cloud.
The edge devices on the user side capture and send their data to the cloud for \emph{prediction services}.
On the one hand, this exchange of data for services has become pervasive since the provider can enhance the user experience by potentially using the data for the betterment of its services~\cite{rana2015data}, which in many cases is offered for free.
On the other hand, as soon as the data is sent to the cloud, it can be misused by the cloud provider, or leaked through security vulnerabilities even if the cloud provider is trusted~\cite{spectre,nytimes2, packetchasing, meltdown, facebook}.
The insight in this paper is that a large fraction of the data is not relevant to the  prediction service and can be sifted prior to sending the data out,
thus enabling access to the services with much greater privacy.
As such, we propose \sieve, an orthogonal approach to the existing techniques that mostly rely on cryptographic solutions and impose prohibitive delays and computational cost.
Table~\ref{table:crypto} summarizes most state-of-the-art encryption-based methods and their runtime compared to unencrypted execution on GPUs.
As shown, these techniques impose between $318\times$ to $14,000\times$ slowdown.
An image classification inference is performed in multiple seconds, an order of magnitude away from the service-level agreement between users and cloud providers, which is between 10 to 100 milliseconds according to MLPerf industry measures~\cite{mlperf:isca20, mlperf}.
Such slowdowns will lead to unacceptable interaction with services that require near real-time response (e.g., home automation cameras).
\sieve provides a middle ground, where there is a provable degree of privacy while the prediction latency is essentially unaffected.
%
%
To that end, \sieve only sends out the features that the provider essentially requires to carry out the requested service.
%
%
%
Existing privacy techniques are applicable to scenarios that can tolerate longer delays, but are not currently suitable for consumer applications, which rely on interactive prediction services.
However, having no privacy protection is also not desirable.

To that end, this paper presents \sieve, a framework that sifts the features of the data based on their relevance to the target prediction task. 
To solve this problem, we reformulate the objective as a \emph{gradient-based} optimization problem, that generates a \emph{sifted representation of the input}.
The intuition is that if a feature can consistently tolerate the addition of noise without degrading the utility, that feature is not conducive to the classification task.
As such, we augment each feature $i$ with a scaled addition of a noise distribution ($\sigma_i . \mathcal{N}(0,1)$) and learn the scales ($\sigma_i$s).
\begin{table*}
\caption{Slowdown of cryptographic techniques vs. conventional GPU execution on Titan Xp and \sieve.\label{table:crypto}}
    \vspace{-2ex}
   \centering
    \begin{tabular}{llllrrrr}
\\\toprule
\multicolumn{1}{l}{Cryptographic} &
  Release &
  \multirow{2}{*}{DNN} &
  \multirow{2}{*}{Dataset} &
  \multicolumn{3}{c}{Prediction Time (sec)} & 
  \multirow{2}{*}{Slowdown} \\ \cmidrule(l){5-7}
  Technique &
  Year &
   &
   &
  Encrypted &
  Conventional &
  \sieve &
  \\\midrule
FALCON~\cite{falcon}           & 2020 & VGG-16    & ImageNet  & 12.96 & 0.0145 & 0.0148 & 906$\times$    \\
DELPHI~\cite{delphi}          & 2020 & ResNet-32 & CIFAR-100 & 3.5 & 0.0112 & 0.0113 & 318$\times$  \\
CrypTen~\cite{crypten} & 2019 & ResNet-18 & ImageNet  & 8.30  & 0.0121 & 0.0123 & 691$\times$    \\
GAZELLE~\cite{gazelle}          & 2018 & ResNet-32 & CIFAR-100 & 82.00 & 0.0112 & 0.0113 & 7,454$\times$  \\
MiniONN~\cite{minionn}          & 2017 & LeNet-5   & MNIST     & 9.32  & 0.0007 & 0.0007 & 14,121$\times$ \\ \bottomrule
\end{tabular}
\end{table*}
To learn the scales, we start with a pre-trained classifier with known parameters and drive a loss function with respect to the scales while the formulation comprises the model as a known analytical function.
The larger the scales, the larger the noise that can be added to a corresponding feature, and the less conducive the feature is.
As such, the learned scales are thresholded to suppress the non-conducive features to a constant value, which yields the sifted representation of the input.
By removing such features, \sieve guarantees that no information about them can be learned or inferred from the sifted representation that the consumer sends.
Figure~\ref{fig:patterns} shows examples of conducive features for multiple tasks discovered by \sieve and the corresponding sifted representation for an example image.
%
%
Our differentiable formulation of finding the scales minimizes the upper bound of the Mutual Information (MI) between the irrelevant features and the sifted representation (maximizing privacy) while maximizing the lower bound of MI between the relevant features and the generated representation (preserving utility).
%
%

\begin{figure}
    \centering
    \includegraphics[width=\linewidth]{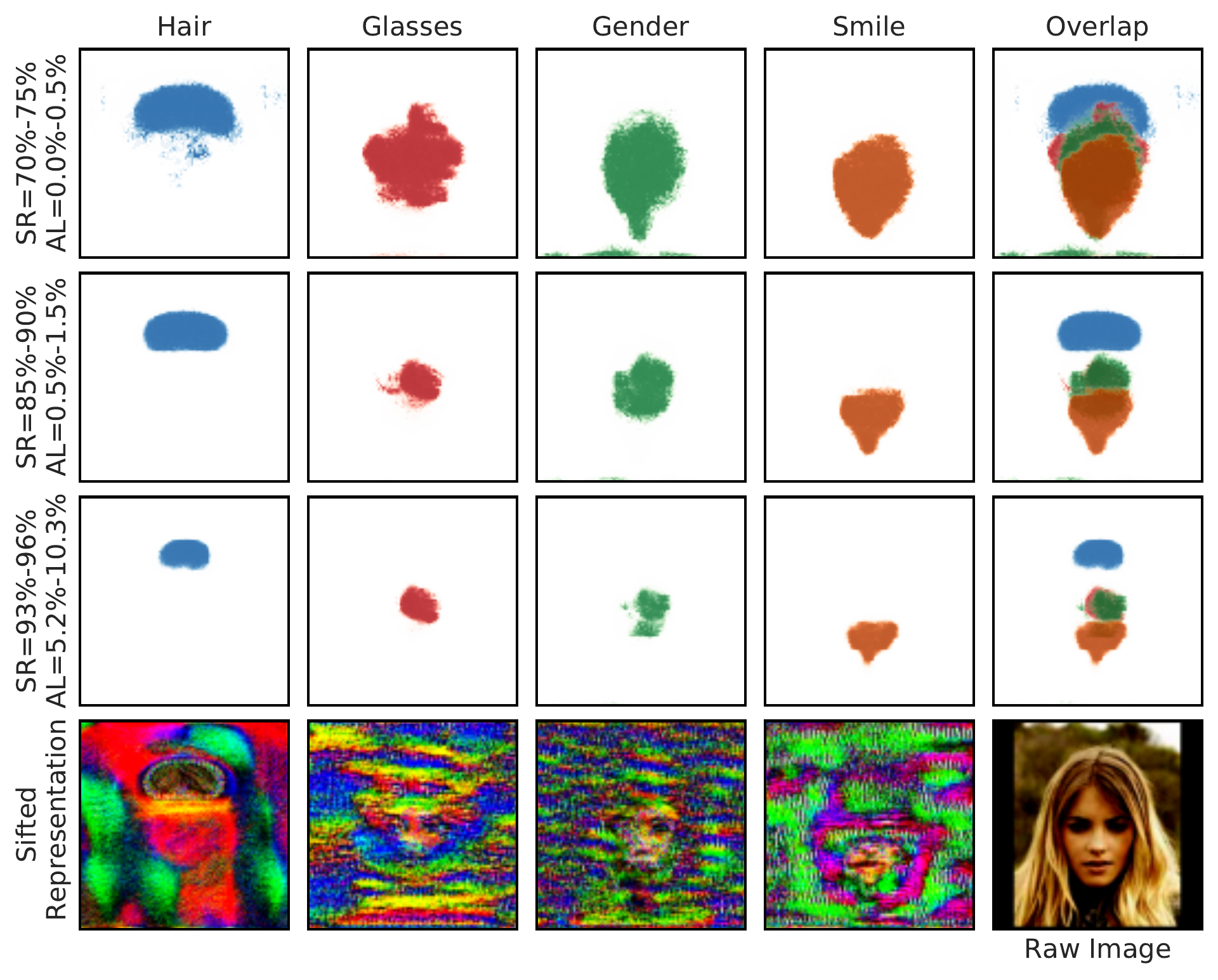}
    \caption{\sieve's discovered features for target DNN classifiers (VGG-16) for black-hair color, eyeglasses, gender, and smile detection. The colored features are conducive to the task. The 3 sets of features depicted for each task correspond to different suppression ratios (SR). AL denotes the range of accuracy loss imposed by the suppression.  
    }
    \label{fig:patterns}
\end{figure}


%
%
%

Experimental evaluation with real-world datasets of UTKFace~\cite{utkface}, CIFAR-100~\cite{cifar100}, and MNIST~\cite{mnist} shows that \cloak can reduce the mutual information between input images and the publicized representation by $85.01\%$ with an accuracy loss of only $1.42\%$.
%
%
%
%
In addition, we evaluate the protection offered by \sieve against adversaries that try to infer data properties from sifted representations on CelebA dataset~\cite{celeba}. 
We show that sifted representations generated for  ``smile detection'' as the target task effectively prevent adversaries from inferring information about hair color and/or eyeglasses.
We show that \sieve can provide these protections even in a black-box setting where we do not have access to the service provider's model parameters or architecture.
Additionally, we show that \sieve outperforms Shredder~\cite{shredder},  a recent work in prediction privacy that heuristically samples and reorders additive noise at run time to imitate the previously collected patterns.
We further show that \sieve can improve the classifier's fairness.
%
%
The code for the proposed method is available at \url{https://github.com/mireshghallah/cloak-www-21}, and the details of the experimental setup and the hyperparameters used for the evaluations are provided in the appendix.
%
%
%
%
%

\section{Preliminaries}

 In this section, we discuss the notation and fundamental concepts used in the rest of the paper, starting with our threat model.


\textbf{Threat Model.}
We assume a remote prediction service setup, where a specific target prediction task is executed on input data. 
Our goal is to create a representation $\mathbf{x_s}$ of the input data $\mathbf{x}$ that has only the features that are essential to the target task, and suppresses excessive features in the input. 
We then send this $\mathbf{x_s}$ to the service provider. 
For our theoretical and empirical evaluations, we adopt supervised classification tasks as our target. 
We assume two access modes to the target classifier $f_{\theta}$: white-box and black-box. In the white-box setup, we assume access to the architecture and parameters $\theta$ of the target classifier. In the black-box setup, we have no access to the target classifier, nor the data it was trained on. 
In both cases, we need labeled training data from the data distribution $\mathcal{D}$, that the target classifier was trained on.
We do not, however, need access to the exact same training data, nor do we need any extra collaboration from the service provider, such as a change in infrastructure or model parameters.

\textbf{Feature Space.} We assume each given input $\mathbf{x}$ to be a collection of features, and group these features based on their importance for the decision making of the target classifier, $f_\theta$.
%
We define the two disjoint feature groups of conducive features, $\mathbf{c}$, which are those relevant to the target task and important to $f_\theta$ and non-conducive features, $\mathbf{u}$, which are less relevant.
Our goal is to find the conducive features and only keep them.
%


%


 
\textbf{Mutual Information.}
 The amount of mutual information between the raw data $\mathbf{x}$, and the representation that is to be publicized, $\mathbf{x_s}$ is a measure of privacy that is widely used in literature~\cite{MI2016, lalitha17, lalitha17b}, and is denoted by $I(\mathbf{x}; \ \mathbf{x_s})$. 
\cloak aims at learning representations $\mathbf{x_s}$ that decrease this mutual information while maintaining the accuracy of the target classification task. 
Formally, \cloak tries to minimize $I(\mathbf{x_s}; \ \mathbf{u})$ while maximizing $I(\mathbf{x_s};\ \mathbf{c})$. 


\section{\sieve's Optimization Problem}\label{sec:problem}

This section formally describes the optimization problem and presents a computationally tractable method towards solving it. 
Let $\mathbf{x}\in\mathbf{R}^n$ be an input,  and $\mathbf{c}\subseteq \mathbf{x}$ and $\mathbf{u}\subseteq \mathbf{x}$ be two disjoint sets of conducive and non-conducive features with respect to our target classifier ($f_{\theta}$).
We construct a noisy representation $\mathbf{x_c}=\mathbf{x}+\mathbf{r}$ where $\mathbf{r}\sim\mathcal{N}(\boldsymbol{\mu}, \boldsymbol{\Sigma})$ and $\boldsymbol{\Sigma}$ is a diagonal covariance matrix, as we set the elements of the noise to be independent.
%
%
%
This noisy representation helps find the conducive features and is used to create a final suppressed representation $\mathbf{x_s}$ that is sent to the service provider. 
The goal is to construct $\mathbf{x_c}$ such that the mutual information between $\mathbf{x_c}$ and $\mathbf{u}$ is minimized (for privacy), while the mutual information between $\mathbf{x_c}$ and $\mathbf{c}$ is maximized (for utility). 
The  is written as the following soft-constrained optimization problem:
\begin{equation}\label{eq:def}
    \min\limits_{\mathbf{x_c}} \quad I(\mathbf{x_c};\ \mathbf{u}) - \lambda I(\mathbf{x_c};\ \mathbf{c}) 
\end{equation}
The intuitive solution is to set $\mathbf{x_c} = \mathbf{c}$. But, directly finding $\mathbf{c}$ is, in most cases, not tractable due to the high complexity of classifiers. 
%
To solve this problem, we bound the terms of our optimization problem of Equation~\ref{eq:def}, and then take an iterative approach~\cite{BBB}. To this end, we find an upper bound for $I(\mathbf{x_c};\ \mathbf{u})$ and a lower bound for $I(\mathbf{x_c};\ \mathbf{c})$.

\subsection{Upper bound on $I(\mathbf{x_c}; \mathbf{u})$} \label{sec:upper}
%
%
%
%
%
Since $\mathbf{u}$ is a subset of $\mathbf{x}$, the following holds:
\begin{equation}\label{eq:upper:0}
    \begin{split}
    I(\mathbf{x_c};\ \mathbf{u}) &\leq I(\mathbf{x_c};\ \mathbf{x}) =  \mathcal{H}(\mathbf{x_c})-\mathcal{H}(\mathbf{x_c}|\mathbf{x}) = \mathcal{H}(\mathbf{x_c})-\frac{1}{2}\log((2\pi e)^n|\boldsymbol{\Sigma}|)
    \end{split}
\end{equation}

Where $\mathcal{H}(\mathbf{x_c}|\mathbf{x})$ is the entropy of the added Gaussian noise. 
%
Here $\boldsymbol{|\Sigma|}$ denotes the determinant of the covariance matrix. Then by applying Theorem~\ref{thm:gaussian-max} (from the appendix) which gives an upper bound for the entropy, to $\mathbf{x_c}$, we can write:
%
%
\begin{equation}\label{eq:upper:3}
     I(\mathbf{x_c};\ \mathbf{u})\leq \frac{1}{2}\log((2\pi e)^n\frac{|Cov(\mathbf{x_c})|}{|\boldsymbol{\Sigma}|})
\end{equation}


Since  $\mathbf{x}$ and $\mathbf{r}$ are independent variables and $\mathbf{x_c} = \mathbf{x}+\mathbf{r}$, we have 
$|Cov(\mathbf{x_c})|=|Cov(\mathbf{x})+\boldsymbol{\Sigma}|$.
%
%
In addition, since covariance matrices are positive semi-definite, we can get the eigen decomposition of $Cov(\mathbf{x})$ as $Q\Lambda Q^T$ where the diagonal matrix $\Lambda$ has the eigenvalues.
Since $\boldsymbol{\Sigma}$ is already a diagonal matrix, $|Cov(\mathbf{x})+\boldsymbol{\Sigma}|=|Q(\Lambda+\boldsymbol{\sigma}^2)Q^T|=\prod_{i=1}^{n}(\lambda_i+\sigma_i^2)$. By substituting this in Equation~\ref{eq:upper:3}, and simplifying we get the upper bound for $I(\mathbf{x_c};\ \mathbf{u})$ as the following:
%
\begin{equation}\label{eq:upper:final}
     I(\mathbf{x_c};\ \mathbf{u})\leq \frac{1}{2}\log((2\pi e)^n \prod_{i=1}^{n}(1+\frac{\lambda_i}{\sigma_i^2}))
\end{equation}


\subsection{Lower bound on $I(\mathbf{x_c}; \mathbf{c})$} \label{sec:lower}


\begin{thm}
The lower bound on $I(\mathbf{x_c};\mathbf{c})$ is:
\begin{equation}  \label{eq:lower}
    \mathcal{H}(\mathbf{c}) + \max\limits_{q}\  \mathbb{E}_{\mathbf{x_c}, \mathbf{c}}[\log\; q(\mathbf{c}|\mathbf{x_c})]
\end{equation}
Where $q$ denotes all members of a possible family of distributions for this conditional probability. 
\end{thm}
\begin{proof}
\vspace{-2ex}
The lemma and accompanying proof for this theorem are in the appendix. 
\vspace{-2ex}
\end{proof}

\subsection{Loss Function}\label{sec:loss}
Now that we have the upper and lower bounds, we can reduce our problem to the following optimization where we minimize the upper bound (Equation~\ref{eq:upper:final})  and maximize the lower bound (Equation~\ref{eq:lower}): 
\begin{equation}
        \min\limits_{\boldsymbol{\sigma},q} \quad \frac{1}{2}\log((2\pi e)^n \prod_{i=1}^{n}(1+\frac{\lambda_i}{\sigma_i^2})) + \lambda \sum_{\mathbf{c_i}, \mathbf{x_{c_i}}} (-\log \; q(\mathbf{c_i}|\mathbf{x_{c_i}}))
\end{equation}

We omit the $\mathcal{H}(\mathbf{c})$ from the lower bound in Equation~\ref{eq:lower}, since it is a constant. We also write the expected value in the same equation in the form of a summation over all possible representations and conducive features.
To make this summation tractable, in our loss function we replace this part of the formulation with the empirical cross-entropy loss of the target classifier over all training examples.
In other words, the loss of preserving the conducive features is substituted by the classification loss for those features. 
We also relax the optimization further by rewriting the first term. Since minimizing this term is equivalent to  maximizing the standard deviation of the noise, we change the fraction into a subtraction. Our final loss function becomes: 
\begin{equation} \label{eq:loss2}
\begin{split}
    \mathcal{L}\!=\!-\log{\frac{1}{n}\sum_{i=0}^{n} \sigma^2_i } +\lambda\,\mathbb{E}_{\mathbf{r}\sim \mathcal{N}(\boldsymbol\mu,\mathbf{\sigma^2}), \mathbf{x}\sim\mathcal{D}}\Big[-\sum_{k=1}^{K} y_k\log\left(f_{\theta}(\mathbf{x} +\mathbf{r})\right)_k\Big]
    \end{split}
    \vspace{-1ex}
\end{equation}

%
%
The second term is the expected cross-entropy loss, over the randomness of the noise and the data instances. 
The variable $\boldsymbol{\mu}$ is the mean of the noise distributions. 
The variable $K$ is the number of classes for the target task, and $y_k$ is the indicator variable that determines if a given example belongs to class $k$. 
More intuitively, the first term increases the noise of each feature and provides privacy. The second term decreases the classification error and maintains accuracy. 
The parameter $\lambda$ is a knob that provides a trade-off between these two. 
%
%
%
%
%
%


\subsection{Suppressed Representation}\label{sec:suppress-th}

After finding the noisy representation $\mathbf{x_c}$, we use it to generate the final suppressed representation $\mathbf{x_s}$. 
By applying a cutoff threshold $T$ on $\boldsymbol\sigma$, we generate binary mask $\mathbf{b}$ such that $b_i = 1$ if ${\sigma_i} \geq T$, and $b_i = 0$ otherwise.
%
We create representation $\mathbf{x_s}= (\mathbf{x} + \mathbf{r}) \circ\mathbf{b} +\mathbf{\mu_s}$, where $\mathbf{r}\sim\mathcal{N}(0,\boldsymbol\sigma)$ and $\mathbf{\mu_s}$ are constant values that are set to replace non-conducive features. 
According to the data processing inequality~\cite{norm2011intuitive}, the upper bound on $I(\mathbf{x_c}; \mathbf{u})$ holds for $\mathbf{x_s}$ as well, since $ I(\mathbf{x_s}; \mathbf{u}) \leq I(\mathbf{x_c}; \mathbf{u})$. 
The same inequality also implies that the lower bound achieved for  $I(\mathbf{x_c}; \mathbf{c})$ does not necessarily hold for $\mathbf{x_s}$.
To address this, we write another optimization problem,
%
%
to find $\mathbf{\mu_s}$ such  that cross entropy loss, i.e,  $\min_{\mathbf{\mu_s}} \sum_{k=1}^{K} y_k\log\left(f_{\theta}(\mathbf{x_s})\right)_k $ is minimized.
Solving this guarantees that the lower bound of Equation~\ref{eq:lower} also holds for $I(\mathbf{x_s};\ \mathbf{c})$. 

%
%
%
%



\section{\sieve Framework} 
\label{sec:train}
This section describes \sieve's framework in more detail.  \sieve comprises of two phases: first, an offline phase where we solve the optimization problems to find the conduciveness of the features and the suppression constant values. 
Second, an online prediction phase where the non-conducive features in a given input are suppressed and a sifted and a suppressed representation of the data is sent to the remote target service provider for prediction. 
In this section we discuss details of these two phases, starting from the details of the offline phase.

\subsection{Noise Re-parameterization and Constraints} \label{sec:constraint}

To solve the optimization problem of Section~\ref{sec:problem}, \sieve's approach is to cast the noise distribution parameters as trainable tensors, making it possible to solve the problem using conventional gradient-based methods. 
To be able to define gradients over the means and variances, we rewrite the noise sampling to be $\mathbf{r} = \boldsymbol{\sigma}\circ\mathbf{e} + \boldsymbol{\mu}$, instead of $\mathbf{r} \sim \mathcal{N}(\boldsymbol{\mu}, \boldsymbol{\sigma^2})$, where $\mathbf{e}\sim\mathcal{N}(0, 1)$.
%
The symbol $\circ$ denotes the element-wise multiplication of elements of $\boldsymbol{\sigma}$ and $\mathbf{e}$.
This redefinition enables us to formulate the problem as an analytical function for which we can calculate the gradients.
We also need to reparameterize $\boldsymbol{\sigma}$ to limit the range of standard deviation of each feature ($\sigma$).
If it is learned through a gradient-based optimization, it can take on any value, while we know that variance can not be negative.
%
%
In addition, we also do not want the $\sigma$s to grow over a given maximum, ${M}$.  
%
%
We put this extra constraint on the distributions, to limit the $\sigma$s from growing infinitely (to decrease the loss), taking the growth opportunity from the standard deviation of the other features.
Finally, we define a trainable parameter $\boldsymbol{\rho}$ and write $\boldsymbol{\sigma} = \frac{1.0 + \tanh(\boldsymbol\rho)}{2} {M}$,
%
where the $\tanh$ function is used to constraint the range of the $\sigma$s, and the addition of $1$ is to guarantee the positivity of the variance. 
%


\subsection{\sieve's Perturbation Training Workflow} \label{sec:train-workflow}

Algorithm~\ref{alg:work} shows the steps of \sieve's optimization process. 
This algorithm takes the training data ($\mathcal{D}$), labels ($y$), a pre-trained model ($f_\theta$), and the privacy-utility knob ($\lambda$) as input, and computes the optimized tensor for noise distribution parameters. 
During the initialization step,  the algorithm sets the trainable tensor for the means ($\boldsymbol\mu$) to 0 and initializes the substitute trainable tensor ($\boldsymbol\rho$) with a large negative number.
This generates the initial value of zero for the standard deviations.
%
%

%
In each step of the optimization, the algorithm calculates the loss function on a batch of training data and computes the gradient of the loss with respect to the $\boldsymbol\mu$ and $\boldsymbol\rho$ by applying backpropagation. 
Since the loss (Equation~\ref{eq:loss2}) incorporates expected value over noise samples, \sieve uses Monte Carlo sampling~\cite{monte} with a sufficiently large number of noise samples to calculate the loss.
This means that to apply a single update to the trainable parameters, \sieve runs multiple forward passes on the entire classifier, at each pass draws new samples for the noise tensor (the elements of which are independently drawn), and averages over the losses and applies the update using the average.
However, in practice, if mini-batch training is used, only a single noise sample for each update can yield desirable results, since a new noise tensor is sampled for each mini-batch.
Once the training is finished, the optimized mean and standard deviation tensors are collected and passed to the next phase.
%

%


%


\subsection{Feature Sifting and Suppression}\label{sec:suppress}

%
%
%
For sifting the features we use the trained standard deviation tensor ($\boldsymbol\sigma$), which we call ``noise map".
A high value in the noise map for a feature indicates that the feature is less important.
%
Different noise maps are created by changing the privacy-utility knob ($\lambda$).
%
%
We use a cutoff threshold $T$, to map the continuous spectrum of values of a noise map, to binary values ($\mathbf{b}$).   
While choosing the cutoff threshold ($T$) depends on the privacy-utility trade-offs, in practice, finding the optimal value for $T$ is not challenging. That is because the trained $\sigma$s are easy to be sifted as they are pushed to either side of the spectrum, i.e., they either have a very large (near $M$) or a very small value (near $0$). See Section~\ref{sec:threshold} for more details. 

%
%
%
%
%


\begin{algorithm}[t]
\caption{  Perturbation Training}
   \label{alg:work}
\begin{algorithmic}[1]
   \STATE {\bfseries Input:} $\mathcal{D}$, $y$, $f_\theta$, $m$, $\lambda$
    \STATE Initialize  $\boldsymbol{\mu}\!=\!0$, $\boldsymbol{\rho}\!=\!-10$ and $M\geq 0$
   \REPEAT
   \STATE Select training batch $\mathbf{x}$ from $\mathcal{D}$ 
   \STATE Sample $\mathbf{e} \sim \mathcal{N}(0,1)$
   \STATE Let $\boldsymbol\sigma = \frac{1.0 + \tanh(\boldsymbol{\rho})}{2} (M)$
   \STATE Let $\mathbf{r} = \boldsymbol{\sigma}\circ\mathbf{e} + \boldsymbol\mu$
   \STATE Take gradient step on $\boldsymbol\mu$, $\boldsymbol\rho$ from Eq.~\eqref{eq:loss2}
   \UNTIL{Algorithm converges}
   \STATE {\bfseries Return:} $\boldsymbol\mu$, $\boldsymbol\sigma$
   
\end{algorithmic}
\end{algorithm}

%
To suppress the non-conducive features, one simple way is to send the noisy representations, i.e,  adding noise from the $(\boldsymbol\mu, \boldsymbol\sigma^2)$ to the input to get the $x_c$ representations that are sent out for prediction.  
This method, however, suffers from two shortcomings: first, it does not directly suppress and remove the features, which could leave the possibility of data leakage. Second, because of the high standard deviations of noise, in some cases, the generated representation might be out of the domain of the target classifier, which could have negative effects on the utility.  
Another way of suppressing the non-conducive features is to replace them with zeros (black pixels in images for example).
This scheme also, suffers from potential accuracy degradation, as the values we are using for suppression (i.e. the zeros) might not match the distribution of the data that the classifier expects.

To address this, we find a suppressed representation (Section~\ref{sec:suppress-th}), i.e., we {train} the constant suppression values that need to replace the non-conducive features.
Intuitively, these learned values reveal what the target classifier perceives as common among all the inputs from the training set, and what it expects to see.
Algorithm~\ref{alg:suppress} shows the steps of this training process.
%
The algorithm finds $\boldsymbol\mu_s$, the values by which we replace the non-conducive features. 
The only objective of this training process is to increase the accuracy, therefore we use the cross-entropy loss as our loss function. 
%


\hfill
\begin{algorithm}[t]
    \caption{ Suppression-Value Training }
   \label{alg:suppress}
\begin{algorithmic}[1]
   \STATE {\bfseries Input:} $\mathcal{D}$, $y$, $f_\theta$, $\boldsymbol\sigma$, $\boldsymbol\mu$, $\mathbf{b}$
    \STATE Initialize  $\boldsymbol{\mu_s}\!=\!\boldsymbol\mu$
   \REPEAT
   \STATE Select training batch $\mathbf{x}$ from $\mathcal{D}$ 
   \STATE Sample $\mathbf{r} \sim \mathcal{N}(0,\boldsymbol\sigma^2)$
   \STATE Let $\mathbf{x_s} = (\mathbf{x}+\mathbf{r})\circ{b} + \boldsymbol{\mu_s}$
   \STATE Take gradient step on $\boldsymbol\mu_s$ from $\mathbb{E}_{r}[\mathcal{L}_{CE}(f_{\theta}(\mathbf{x_s}),\ y)]$
   \UNTIL{Algorithm converges}
   \STATE {\bfseries Return:} $\boldsymbol\mu_s$
   
\end{algorithmic}
\end{algorithm}



\subsection{Online Prediction}

The prediction (inference) phase is when unseen test inputs that we protect are sent to the remote service provider for classification. 
This process is computationally efficient; it only adds noise sampling, masking, and addition to the normal conventional prediction process.
First, a noise tensor sampled from the optimized distribution $\mathcal{N}(0, \boldsymbol\sigma^2)$ is added to the input, then the binary mask $b$ is applied to the noisy input image. Finally, $\boldsymbol\mu_s$ is added to $\mathbf{x}$ and the resulting sifted representation is sent to the service provider.
As an example, the last row of Figure~\ref{fig:patterns} shows representations generated by \sieve, for different tasks, using the noise maps from the third row. As the images show, the non-conducive features are removed and replaced with $\boldsymbol\mu_s$. The conducive features, however, are visible.

%

 
%

%

\section{Experimental Results}
\label{sec:exps}
To evaluate \cloak, we use four real-world datasets on four Deep Neural Networks (DNNs). Namely, we use VGG-16~\cite{vgg} and ResNet-18~\cite{resnet} on CelebA~\cite{celeba}, AlexNet~\cite{alexnet} on CIFAR-100~\cite{cifar100}, a modified version of VGG-16 model on UTKFace~\cite{utkface}, and LeNet-5~\cite{lenet5} on MNIST~\cite{mnist}. 
The mutual information numbers reported in this section are estimated over the test set using the Shannon Mutual Information estimator provided by the Python ITE toolbox~\cite{itetoolbox}.
%
%
For the experiments that are devised to compare \sieve with previous work, Shredder~\cite{shredder}, in order to create a similar setup, we apply \sieve to the last convolution layer of the DNN and create \textit{sifted intermediate representations} which are then sent to the target classifier. 
In the other experiments, \sieve is applied directly to the input images. 
%
%
Code and information about hyper-parameters used in each of the experiments is provided in the appendix.
\subsection{Detailed Experimental Setup}
In this section, we elaborate on the details of our experimental setup. This includes dataset specifications, hardware and OS specifications, neural network architectures, and finally, mutual information estimation. 
\subsubsection{Dataset Specifications}
There are four datasets used in our evaluations: CelebA~\cite{celeba}, CIFAR-100~\cite{cifar100}, UTKFace~\cite{utkface} and MNIST~\cite{mnist}. We have used these datasets with VGG-16~\cite{vgg}, ResNet-18~\cite{resnet}, AlexNet~\cite{alexnet}, VGG-16 (modified), and LeNet-5~\cite{lenet5} neural networks, respectively.
We define a set of target prediction tasks  over these datasets. 
Specifically, we use smile detection, black-hair color classification, and eyeglass detection on CelebA, the 20 super-class classification on CIFAR-100, and gender detection on UTKFace. For MNIST, we use a classifier that detects if the input is greater than five and another one that classifies what the input digit actually is.  
The accuracy numbers reported in this section are all on a held-out test set, which has not been seen during training by the neural networks.
For \cloak results, since the output is not deterministic, we repeatedly run the prediction ten times on the test set with the batch size of one and report the {mean accuracy}.
Since the {standard deviation} of the accuracy numbers is small (consistently less than $1.0\%$) the confidence bars are not visible on the graphs. 
The input image sizes for CelebA, CIFAR-100, UTKFace and MNIST are $224\times224\times3$, $32\times32\times3$, $32\times32\times3$, and $32\times32$, respectively.
In addition, in our experiments, the inputs are all normalized to 1.
The experiments are all carried out using Python 3.6 and PyTorch 1.3.1.
We use Adam optimizer for perturbation training. 
%
%
%
%
%
%
%
%
\subsubsection{Experimentation Hardware and OS}
We have run the experiments for CelebA dataset on an Nvidia RTX 2080 Ti GPU, with 11GB VRAM, paired with 10 Intel Core i9-9820X processors with 64GBs of memory. The rest of the experiments were run on the CPU. The system runs an Ubuntu 18.04 OS, with CUDA version V10.2.89. 

\subsubsection{Neural Network Architectures}

%
%

The code for all the models is available in the supplementary materials. The VGG-16 for UTKFace is different from the conventional one in the size of the last 3 fully connected layers. They are (512,256),  (256,256) and (256,2). 
The pre-trained accuracy of the networks for smile detection, super-class classification, gender detection, and greater than five detection are $91.8\%$, $55.7\%$, $87.87\%$, and $99.29\%$.

\subsubsection{Mutual Information Estimation}
The mutual information between the input images and their noisy representations are estimated over the test set images using ITE~\cite{itetoolbox} toolbox's Shannon mutual information estimator. For MNIST images, our dataset has inputs of size $32\times32$ pixels, which we flatten to $1024$ element vectors, for estimating the mutual information. For other datasets, since the images are larger ($32\times32\times3$), there are more dimensions and mutual information estimation is not accurate. So, we calculate mutual information channel by channel (i.e. we estimate the mutual information between the red channel of the image and its noisy representation, then the green channel and then blue), and we average over all channels. 

The numbers reported in ~\ref{sec:res2} are mutual information loss percentages, which means the lost mutual information among the publicized image and the original one is divided by the information content in the original images. This information content was estimated using self-information (Shannon information), using the same toolbox.

\subsection{Privacy-Accuracy Trade-Off} \label{sec:res2}

 \begin{figure*}
    \centering
        \begin{subfigure}{0.3\textwidth}
     \includegraphics[width=\linewidth]{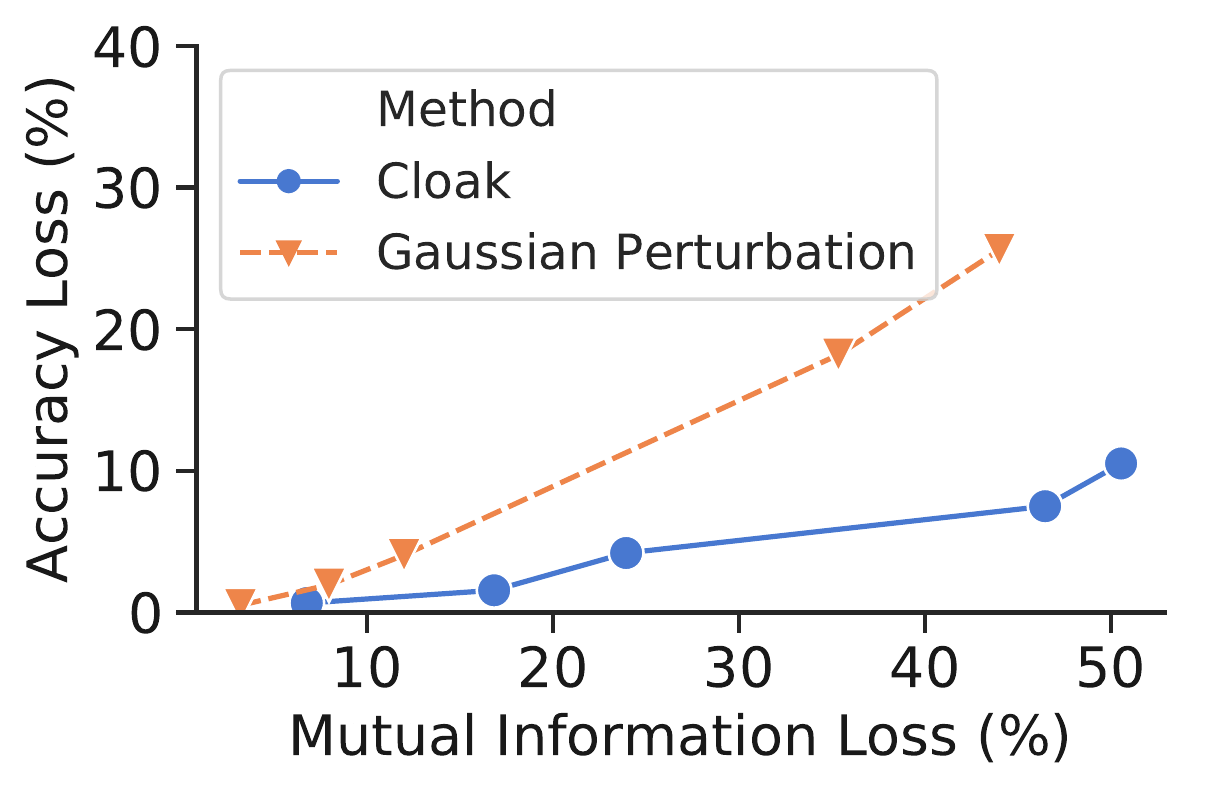}
     \caption{CIFAR-100}
     \label{fig:mi-cifar100}
    \end{subfigure}
    \begin{subfigure}{0.3\textwidth}
     \includegraphics[width=\linewidth]{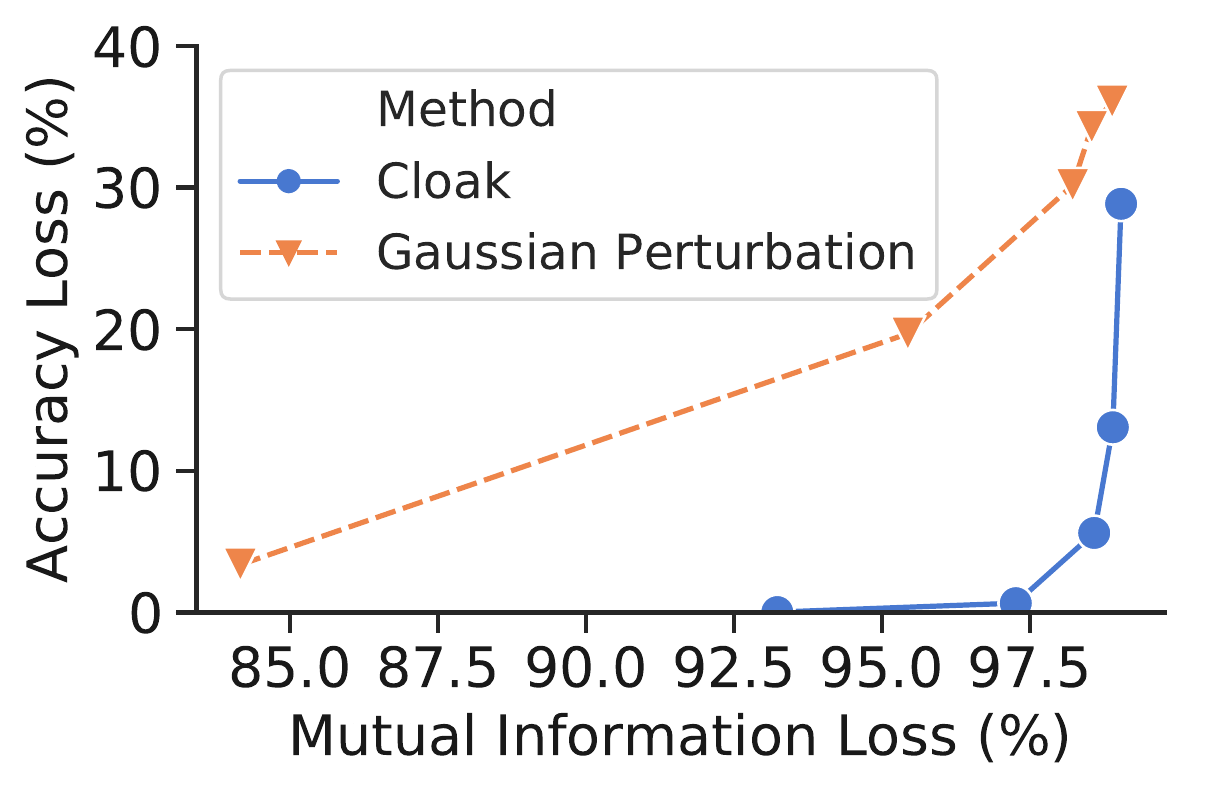}
     \caption{MNIST}
     \label{fig:mi-lenet}
    \end{subfigure}
    \begin{subfigure}{0.3\textwidth}
     \includegraphics[width=\linewidth]{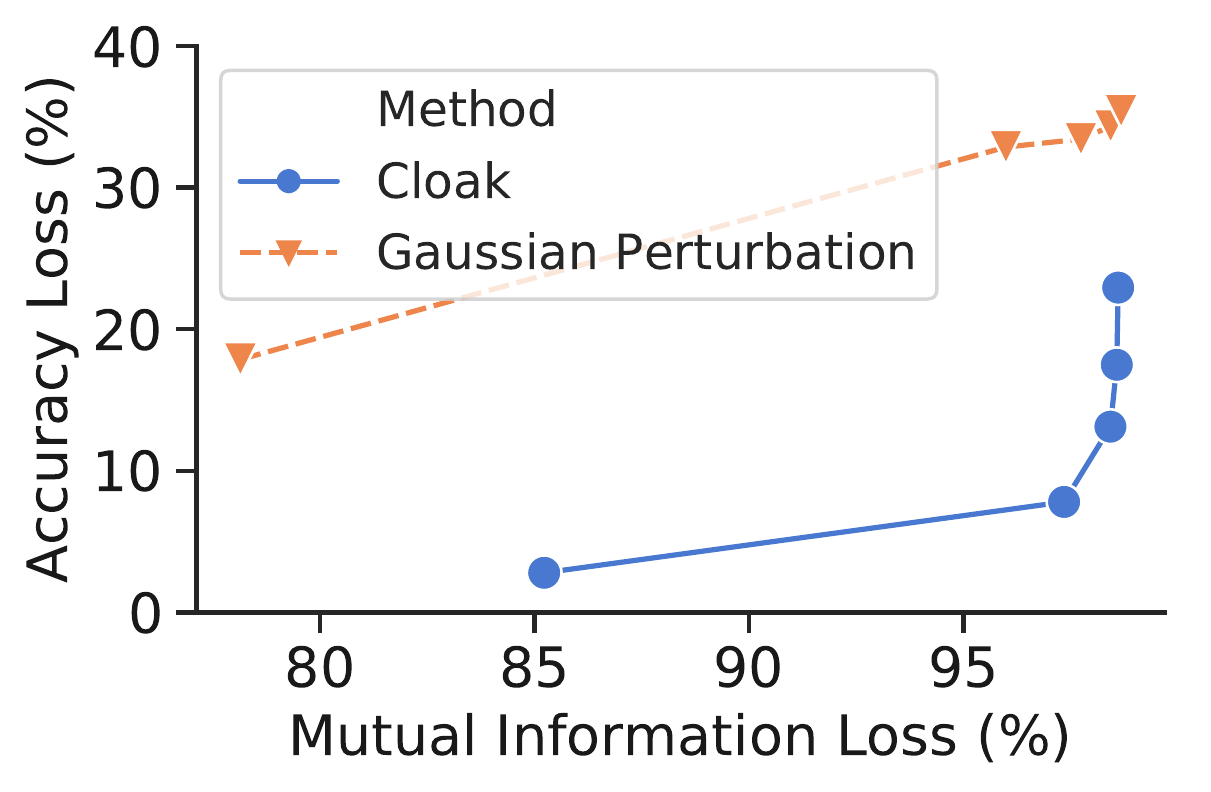}
     \caption{UTKFace}
     \label{fig:mi-xx100}
    \end{subfigure}
    \caption{Privacy-accuracy trade-off for CIFAR-100, MNIST and UTKFace dataset.}
    \label{fig:mi}
\end{figure*}
Figure~\ref{fig:mi} shows accuracy loss of the DNN classifiers using sifted representations vs. the loss in mutual information. 
This is  the loss in mutual information between the original image and its noisy representation, divided by the amount of information in bits in the original image.
The target tasks are 20 superclass classification for CIFAR-100, $>5$ classification for MNIST and gender classification for UTKFace.
In this experiment, we compare \sieve to adding Gaussian perturbation of mean zero and different standard deviations to all pixels of the images. For fair comparison, we choose \sieve's suppression with noisy representations.
%
%
%
%
For MNIST and UTKFace, \sieve reduces the information in the input significantly ($93\%$ and $85\%$ respectively) with little loss in accuracy ($0.5\%$ and $2.7\%$). 
In CIFAR-100, the accuracy is slightly more sensitive to the mutual information loss. 
%
%
This is due to the difference in the classification tasks. The tasks for MNIST and UTKFace have only two classes, while for CIFAR-100, the classifier needs to distinguish between 20 classes.

For all three datasets, we see that \sieve achieves a  significantly higher accuracy for same loss in mutual information compared to Gaussian perturbation.
This is because \sieve adds more noise to the irrelevant features, and less to the relevant ones, whereas Gaussian perturbations are added uniformly across the input. 
We do not present mutual information results for the CelebA dataset here, since the input images have an extremely large number of features and the mutual information estimator tool is not capable of estimating the mutual information accurately. 
%
\subsection{Adversary to Infer Information}
\label{sec:adv}

To further evaluate the effectiveness of the representations that \sieve generates, we devise an experiment in which an adversary tries to infer properties of the sifted representations using a DNN classifier.
We assume two adversary models here. First, the adversary has access to a unlimited number of samples from the sifted representations, therefore she can re-train her classifier to regain accuracy on the sifted representations. 
Second, a model in which the adversary's access to the sifted representation is limited and therefore she cannot retrain her classifier on the sifted representations. 
In this experiment, we choose smile detection as the target prediction task for which \sieve generates representations. 
Then, we model adversaries who try to discover two properties from the sifted representations: whether people in images wear glasses or not and whether their hair is black or not.  
The adversaries have pre-trained classifiers for both these tasks. The classifiers are VGG-16 DNNs, with accuracy of  $96.4\%$ and $88.2\%$ for glasses and hair color classification, respectively.
%
%

Figure~\ref{fig:adversary} shows the results of this experiment.
%
Each point in this figure is generated using a noise map with a Suppression Ration (SR) noted in the figure.
Higher SR means more features are suppressed.
%
%
%
%
%
When adversaries do not retrain their models, using  sifted representations with $95.6\%$ suppression ratio causes the adversaries to almost completely lose their ability to infer eyeglasses or hair color and reach to the random classifier accuracy ($50\%$). This is achieved while the target smile detection task only loses $5.16\%$ accuracy.
%
%
When adversaries retrain their models, using representations with slightly higher suppression ratio ($98.3\%$) achieves the same goal. But this time, the accuracy of the target task drops to  $78.9\%$.
%
%
%
With the same suppression ratio, the adversary who tries to infer hair color loses more accuracy than the adversary who tries to infer eyeglasses. 
This is because, as shown in Figure~\ref{fig:patterns}, the conducive features of smile overlap less with the conducive features of hair than with the conducive features of eyeglasses. 
%
%

\begin{figure}
    \centering
    \includegraphics[width=0.99\linewidth]{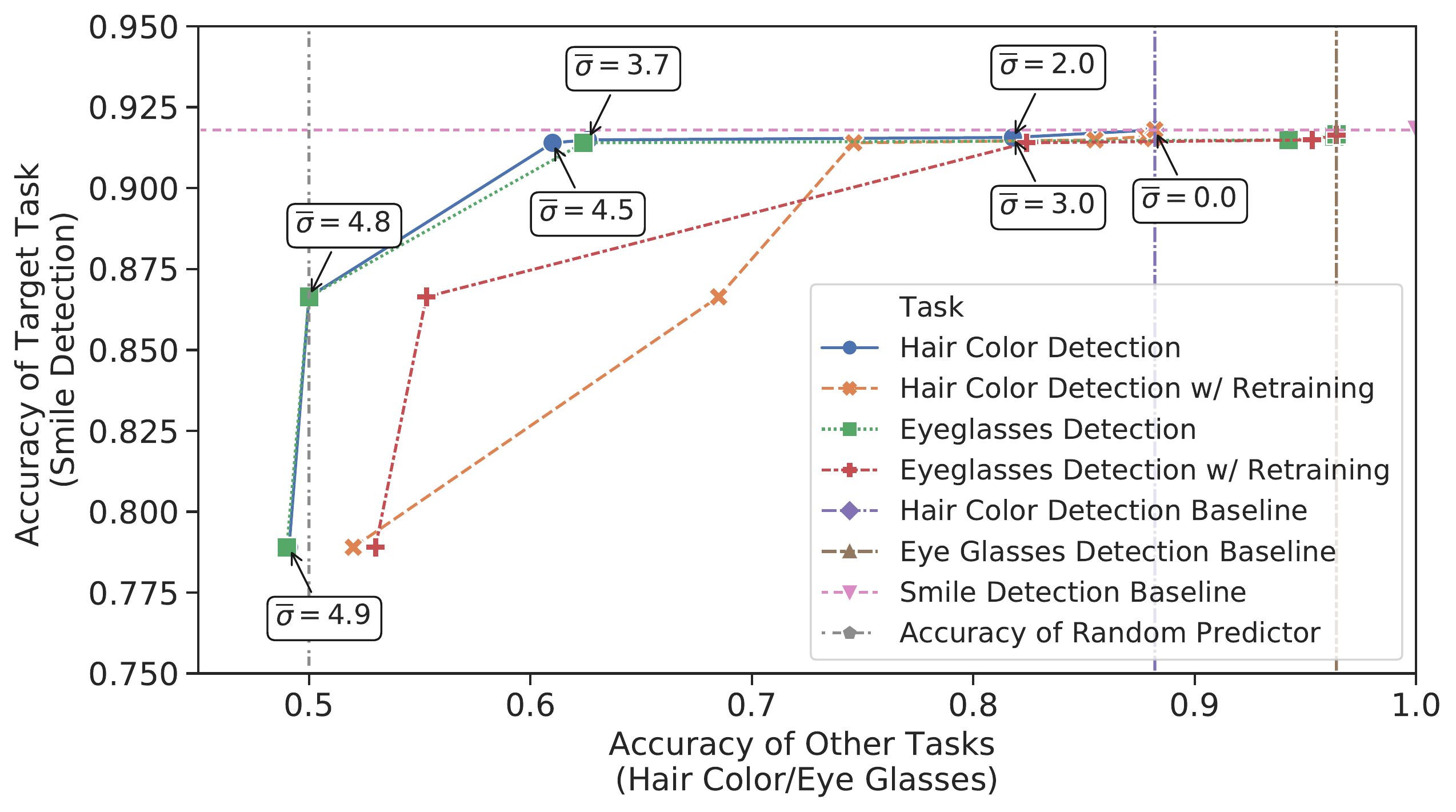}
    \caption{\sieve's protection for target task of smile detection (CelebA dataset) against adversaries that try to infer black-hair color or wearing of eyeglasses from the sifted representations. }
    \vspace{-0.1ex}
    \label{fig:adversary}
\end{figure}
\begin{figure}
    \centering
    \includegraphics[width=0.99\linewidth]{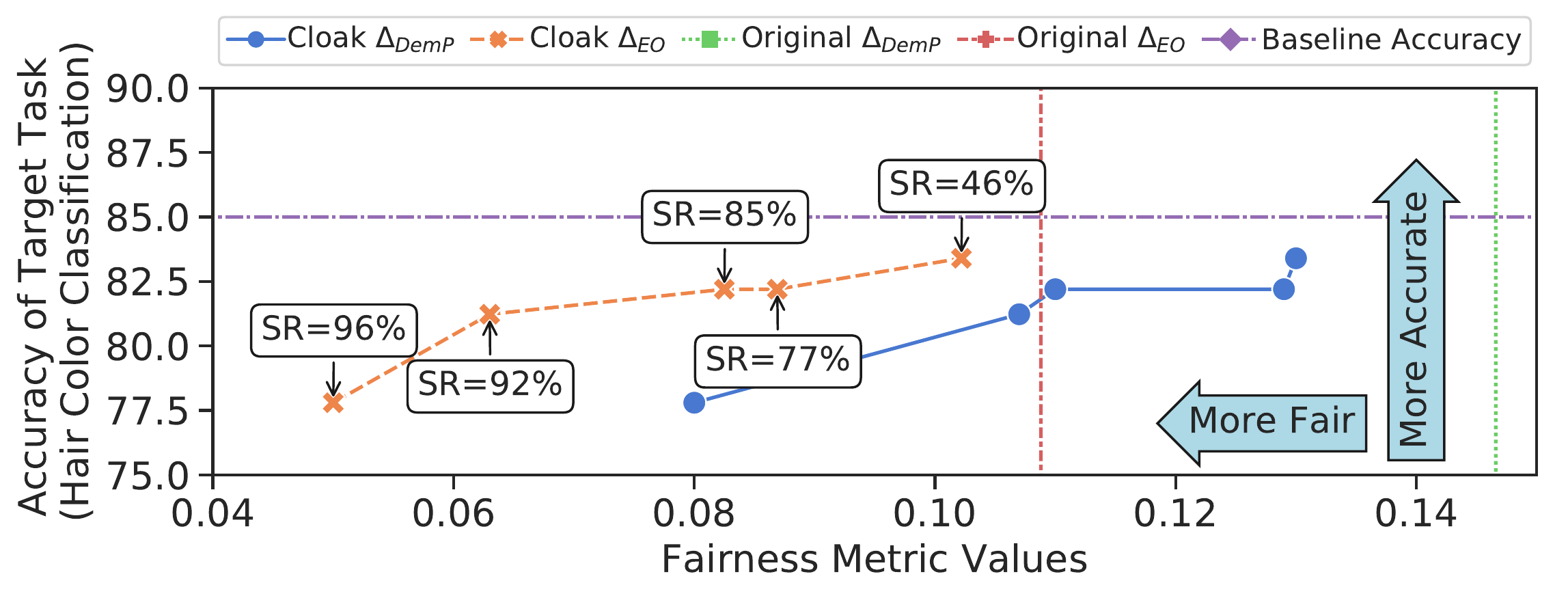}
    \caption{Effects of \sieve on fairness}
    \vspace{-0.1ex}
   \label{fig:fairness}
\end{figure}
\begin{figure}
    \centering
   \includegraphics[width=0.8\linewidth]{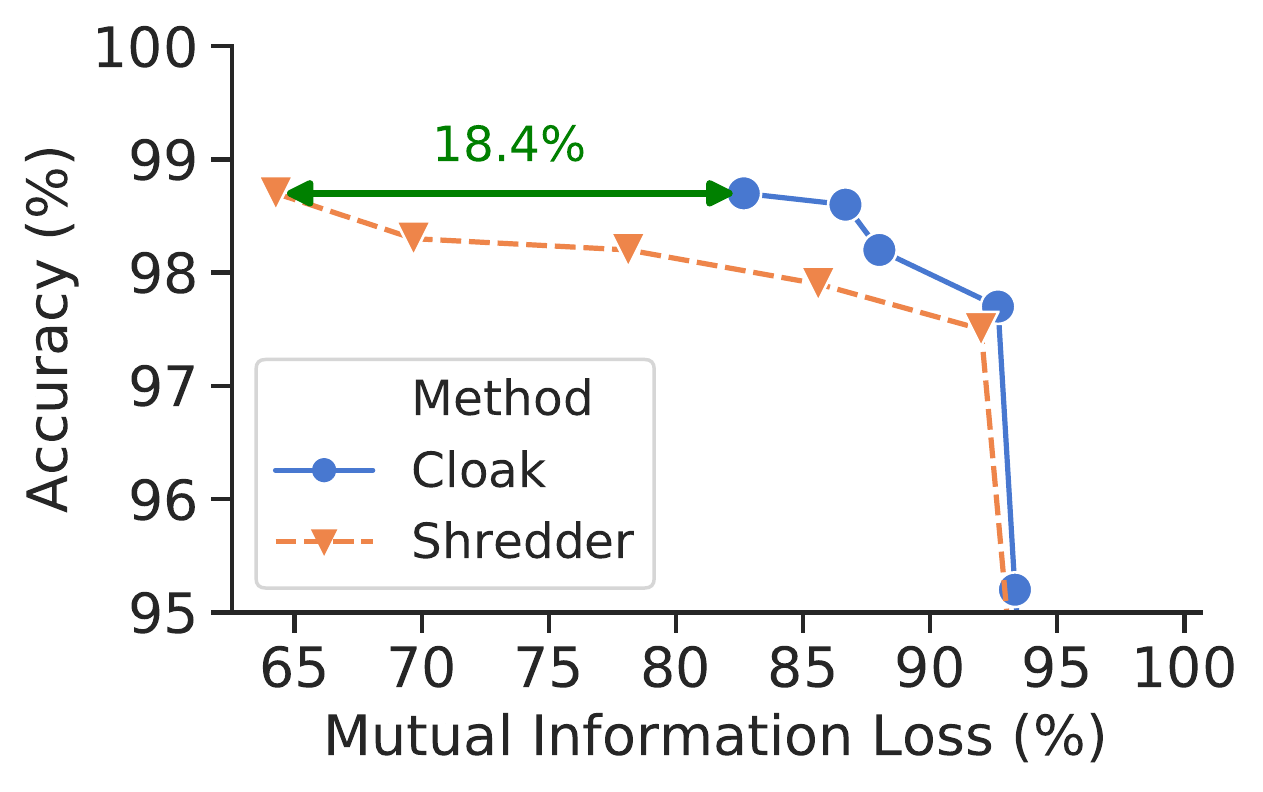}
    \caption{Comparison to Shredder~\cite{shredder}}
    \vspace{-0.1ex}
   \label{fig:shredder}
\end{figure}


\subsection{Black-Box Access Mode}\label{sec:black-box}
To show the applicability of \sieve, we show that it is possible for \sieve to protect users' privacy even when we have limited access to the target model. 
We consider a black-box setting in which we assume \sieve does not have any knowledge of the target model architecture or its parameters and is only allowed to send requests and get back responses. 
In this setting, we first train a substitute model that helps us to train \sieve's representations. 
Note that training a substitute model for black-box setting is a well-established practice in the context of adversarial examples~\cite{papernot2017practical, lu2017safetynet} and inference attacks~\cite{attack1,memgaurd}.
The main challenge is generating the training data needed for training the substitute model. 
However, that has been already addressed in previous work and we follow a similar methodology to the methodology described in Shokri et al.~\cite{attack1}.  
We divide the original dataset (CelebA) into two equal-size disjoint training sets, one for the target and the other for the substitute model.
%
%
%
%
%
%
%
%
%
We assume a target service provider that has two ResNet18~\cite{resnet} DNNs deployed, one for the task of black hair color classification, and one for smile detection. 
Since we assume no knowledge of the model architecture, \sieve substitutes the target classifiers with another architecture, i.e, with two VGG-16 DNNs. 
%
%
\sieve substitute models for the hair and smile tasks have accuracies of $84.9\%$ and $90.9\%$ and the target models have accuracies of $87.3\%$ and $91.8\%$. 
After training the substitute model, we apply \sieve to them to find noise maps and suppressed representations.

Figure~\ref{fig:blackbox-smile} and ~\ref{fig:blackbox-hair} show the results for these experiments. 
\sieve performs similarly effective in both white-box and black-box settings and for both hair color classification and smile detection tasks.
The reason is that the DNN classifiers of the same task are known to learn similar patterns and decision boundaries~\cite{papernot2017practical, memr1}.
For the smile detection, we can see that with suppression ratio of $33\%$, The \sieve black-box generated representations can get prediction accuracy of $91.3\%$, even higher than the baseline prediction accuracy of the classifier it is produced from. That is because the generated representations are fed to the target classifier, which has a higher baseline accuracy than the substitute model.

\subsection{Post-hoc Effects of \sieve on Fairness}
\sieve, by removing extra features, not only  benefits privacy but can also remove unintended biases of the classifier, resulting in a more fair classification.
In many cases the features that bias the classifiers highly overlap with the non-conducive features that \sieve discovers. 
Therefore, applying \sieve can result in predictions that are more fair, without the need to change the classifier.
This subsection evaluates this positive side-effect of \sieve by adopting a setup similar to that of Kairouz et al.~\cite{kairouz2019censored}.
%
%
%
We measure the fairness of the black-hair color classifier using the sifted representations, while considering gender to be a sensitive variable that can cause bias.
%
%
We use two metrics for our experiments, the difference in Demographic Parity ($\Delta_{DemP}$), and the difference in Equal Opportunity ($\Delta_{EO}$). More details on the metrics and the measurements can be found in the supplementary material.
%
%
Figure~\ref{fig:fairness} shows that as \sieve suppresses more non-conducive features, the fairness metrics improve significantly. 
We see $0.05$ reduction in both metrics due to the removal of gender related non-conducive features. 
%
%
%
It is noteworthy that the biasing features in the hair color classifier are not necessarily the gender features shown in Figure~\ref{fig:patterns}.
Those features show what a gender classifier uses to make its decision.

\begin{figure*}
    \centering
    \begin{subfigure}{0.24\textwidth}
        \includegraphics[width=\linewidth]{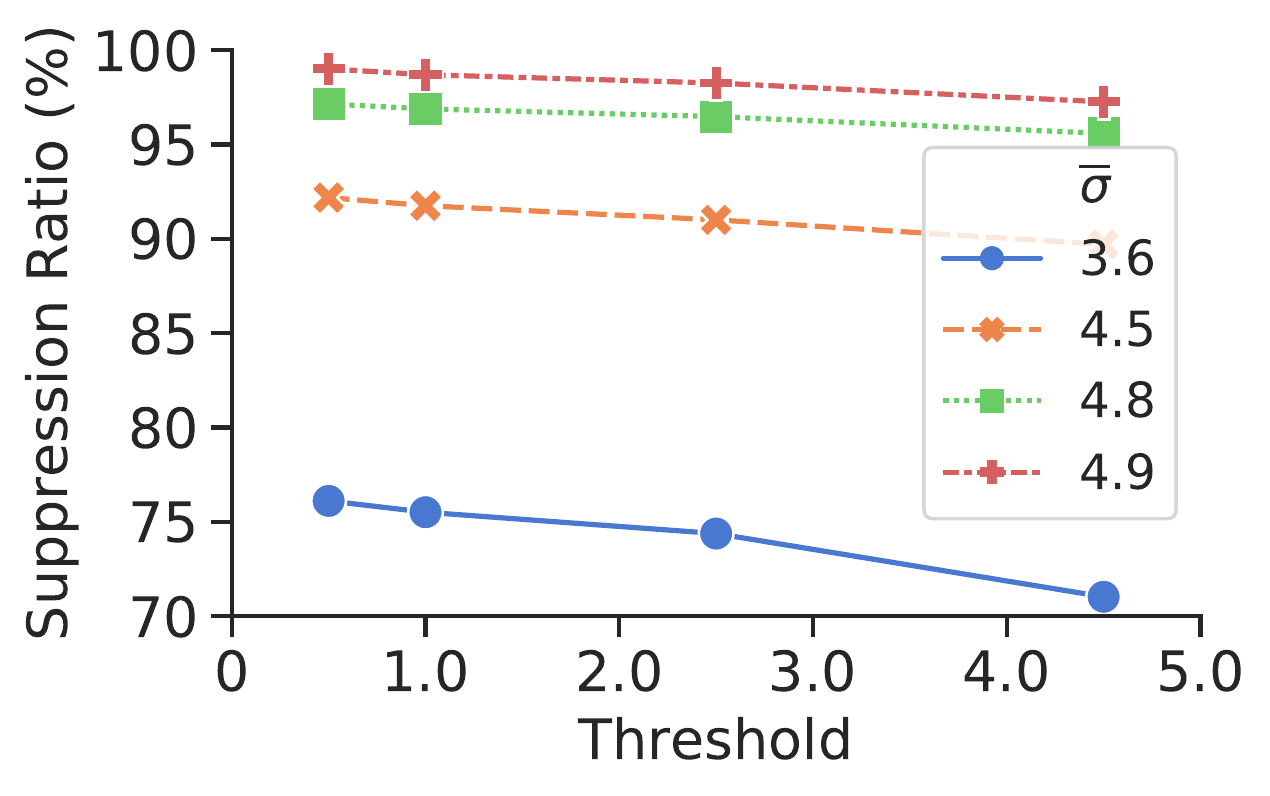}
        \caption{Threshold Sensitivity}
        \label{fig:threshold}
        \end{subfigure}
    \begin{subfigure}{0.24\textwidth}
        \includegraphics[width=\linewidth]{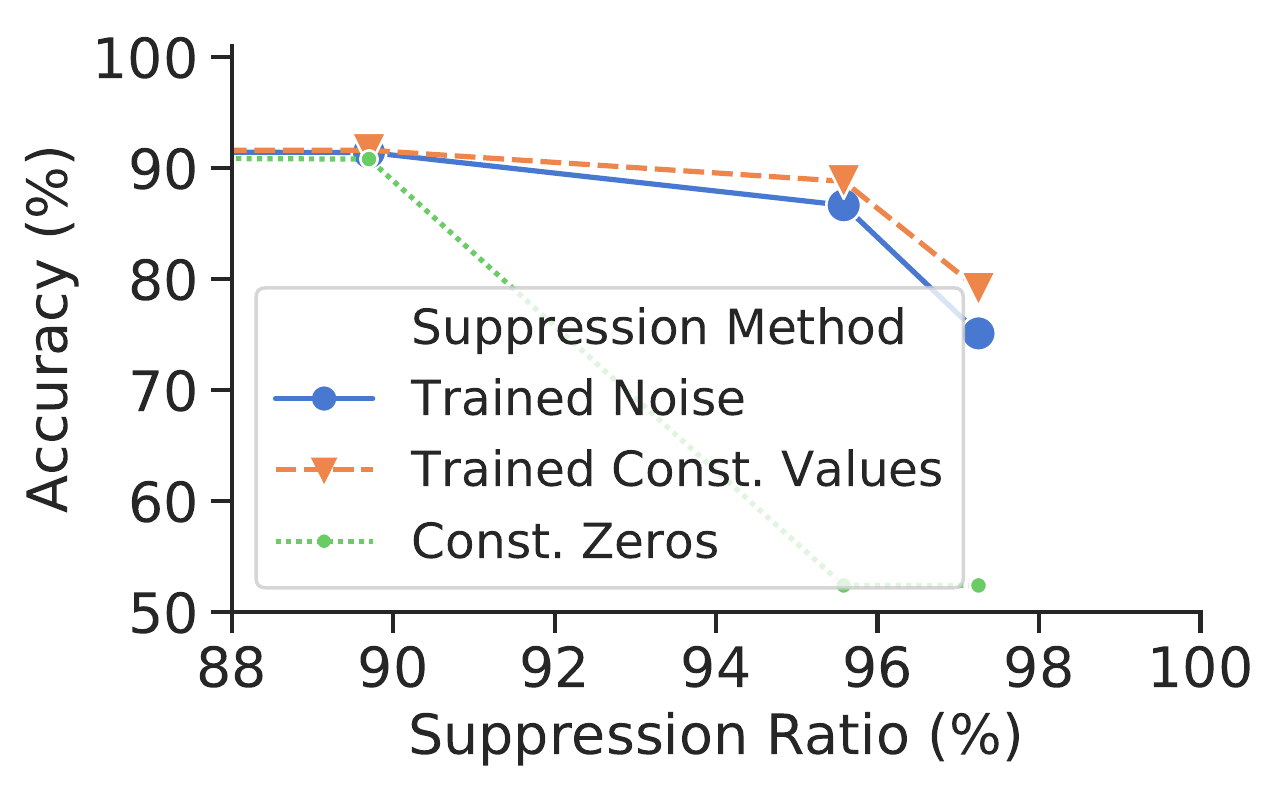}
        \caption{Suppression Schemes}
        \label{fig:suppression}
    \end{subfigure}
        \begin{subfigure}{0.24\textwidth}
        \includegraphics[width=\linewidth]{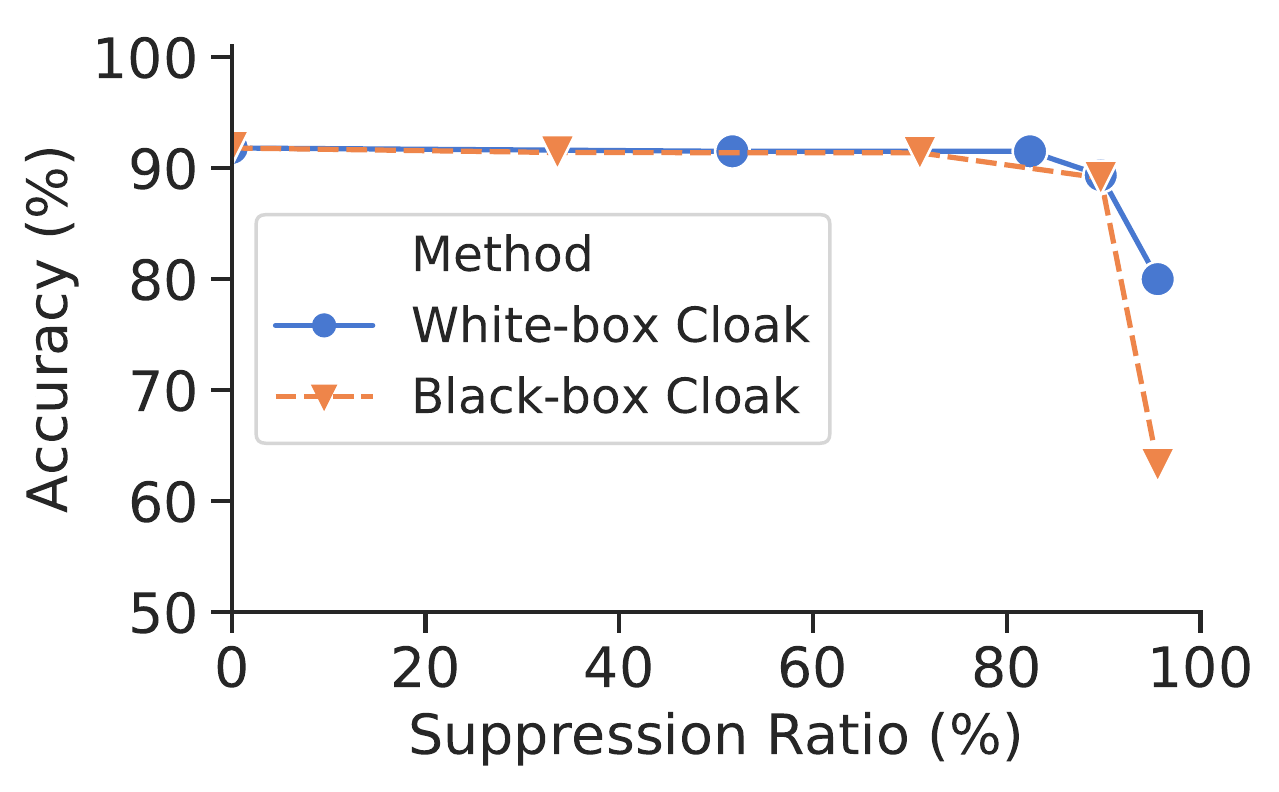}
        \caption{Black-box-smile detection}
        \label{fig:blackbox-smile}
        \end{subfigure}
        \begin{subfigure}{0.24\textwidth}
        \includegraphics[width=\linewidth]{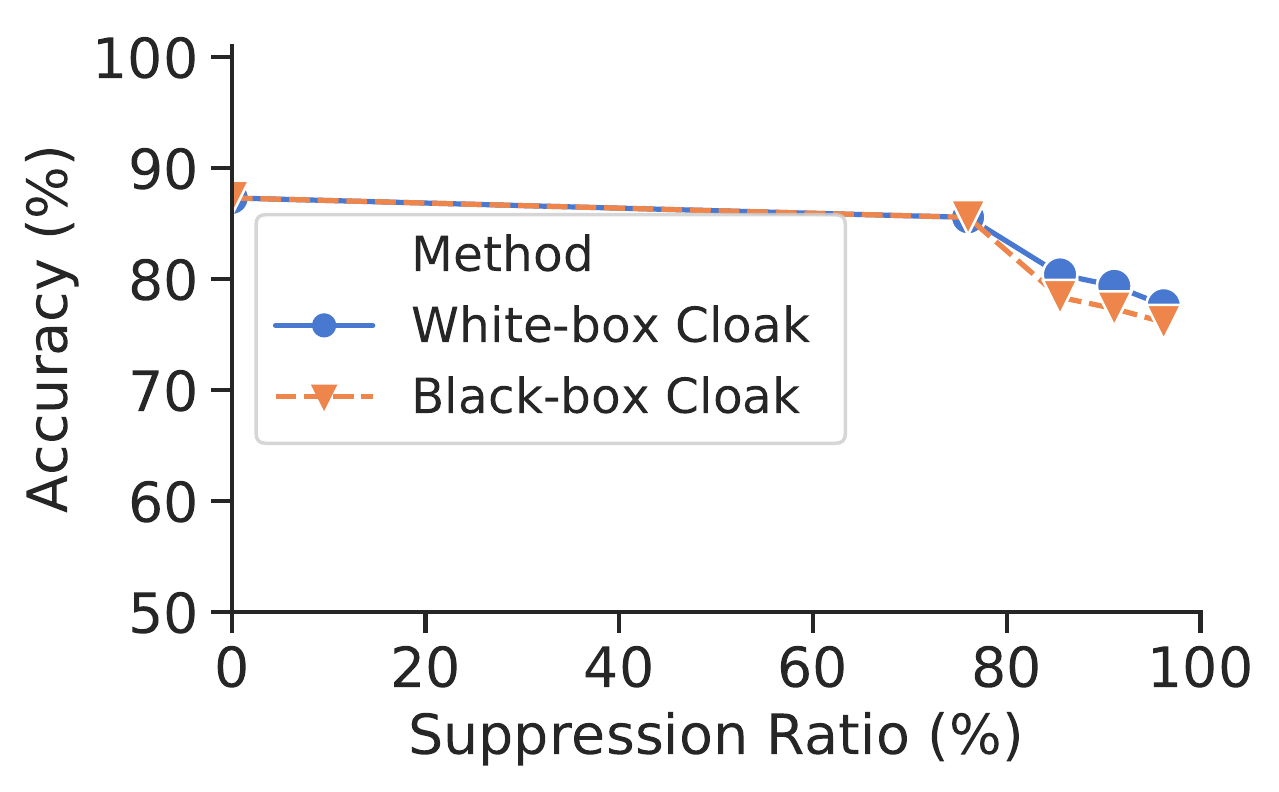}
        \caption{Black-box-hair color}
        \label{fig:blackbox-hair}
        \end{subfigure}
        
    \caption{ (\subref{fig:threshold}) shows the effect of different thresholds on suppression rate. (\subref{fig:suppression}) compares different suppression methods. (\subref{fig:blackbox-smile}) and (\subref{fig:blackbox-hair}) show performance of \sieve in a black-box setting.}
\end{figure*}

\subsection{Thresholds, Suppression Mechanisms, and Comparison to Shredder}
\label{sec:threshold}
\textbf{Sensitivity to threshold values.} Figure~\ref{fig:threshold} shows the effect of  different thresholds ($T$) values on suppression ratio of features on smile detection (on CelebA/ VGG-16). 
%
Different series show different noise maps attained with different values of $\lambda$.
$\bar{\sigma}$ denotes the average standard deviation of a noise map, and the parameter $M$ (maximum standard deviation) of Section~\ref{sec:constraint} is set to 5. 
The figure shows that the choice of $T$ is not critical and in fact is a simple task, since it has little effect on the subset of features that get suppressed. 
This is because during the training of perturbation parameters, the standard deviations are pushed to the either sides of the spectrum (close to $0$ or close to $M$). 

\textbf{Different suppression schemes.} Figure~\ref{fig:suppression} shows the accuracy of three suppression schemes described in Section~\ref{sec:suppress} on the smile detection task (on CelebA/ VGG-16). 
Among different schemes, suppression using the trained values yields better accuracy for the same suppression ratio, since it captures what the classifier expects to receives.
%
Suppression with noise (sending noisy representations) performs slightly worse than training, and that is mainly due to the uncertainty brought by the noise.

%
%

%
%

\textbf{Comparision to Shredder.} Figure~\ref{fig:shredder} compares \sieve and Shredder~\cite{shredder} on the MNIST dataset using LeNet for the target task of digit classification. To create a fair setup, we deploy \sieve to the output of the last convolution layer of LeNet, similar to Shredder. 
\sieve achieves a significantly higher accuracy for same levels of MI loss, which shows the effectiveness of \sieve, in the intermediate representation space. 
For the initial point where there is almost no loss in accuracy, \sieve achieves $18.4\%$ more information loss. 
This better performance is partly due to directly learning the importance of each feature, as opposed to generating patterns similar to a collection that yields high accuracy.
It is also partly due to the extra step that \sieve takes at learning the constant suppression values, which ensures the generated representations are in the domain of the classifier.

\section{Related Work}
\label{sec:related}

This section reviews related work on the privacy of web services. The section first briefly discusses the privacy of web applications in general, and then more thoroughly discusses privacy in the context of machine learning. 

\subsection{Web-application Privacy}

Despite the privacy issues, sharing personal content on the web unfortunately is still common. Therefore, researchers lavished attention on the research that makes such sharing safe, secure, and private~\cite{more1, more2}. Mannan et al.~\cite{2008priv} proposed a method that focuses on privacy-enhanced web content sharing in any user-chosen web server. 
There is also a body of work that conducts longitudinal studies on deleted web content and their subsequent information leakage~\cite{soups, del1}. 
The research in this area focuses on data leakage through social media~\cite{2009www, sajadmanesh2017kissing}, blogging services that publish information~\cite{2019ccsasia}, or aggregation of web data~\cite{2019www}. 
\cloak, however, focuses on an inference-as-a-service setup where private queries that potentially contain sensitive information are sent to a web-service to run machine learning inference.

\subsection{Machine Learning Privacy}
Privacy-preserving machine learning research can be broadly categorized based on the phase on which they focus, i.e., training vs prediction.
The majority of these studies fall under the training category~\cite{survey} where they try to protect contributors' private data from getting embedded in the trained ML model~\cite{shokriDNN,abadiDNN,DPPCA,chaudhuriERM,chaudhuriPCA, pate, Lim2018LearningLE} or from being published in public datasets~\cite{dwork06tcc,dwork06euro,dwork14book}.
However, the impending importance of prediction (inference) privacy has led to the emergence of recent research efforts in this direction~\cite{osia2,osia1,notjust,dowlin16icml,leroux2018privacy, he2019model}.
There is also a smaller body of work focused on the privacy of model architecture and parameters~\cite{carlini2020cryptanalytic,krishna2019thieves}, which is out of the scope of this paper.
%
%
Below, the more related works are discussed in more detail.

\textbf{Training phase.} For \emph{training}, the literature abounds with studies that use noise addition as a randomization mechanism to protect privacy~\cite{chaudhuriERM, chaudhuriPCA,dwork14book, abadiDNN,shokriDNN,pate,Papernot2018ScalablePL}. 
Most notably, differential privacy~\cite{dwork06tcc}, a mathematical framework that quantifies privacy, has spawned a vast body of research in noise-adding mechanisms. 
For instance, it has been applied to many machine learning algorithms, such as logistic regression~\cite{chaudhuriLogistic}, statistical risk minimization~\cite{chaudhuriERM}, principal component analysis~\cite{DPPCA,chaudhuriPCA}, and deep learning~\cite{shokriDNN,abadiDNN, pate, Papernot2018ScalablePL, sajadmanesh2020differential}, to name a few. 
Many of these studies have applied differential privacy to a training setting where they are concerned with leaking private information in training set through the machine learning model.  
There is also a body of work focused on secure training of machine learning models using cryptographic protocols~\cite{ secureml, agrawal2019quotient, ryffel2020ariann, ryffel2019partially, agrawal2019quotient, hashemi2020darknight}.

Finally, there are also several privacy-enhancing mechanisms, such as Federated learning~\cite{liu2019enhancing, flsurvey} and Split learning~\cite{Poirot2019SplitLF, Singh2019DetailedCO}, which use gradients or abstract representations of data in lieu of raw inputs, to train ML models and enhance privacy. These methods have been coupled with differential privacy~\cite{balle2020privacy, ramaswamy2020training, fl-secureagg} or information-theoretic notions~\cite{Vepakomma2020NoPeekIL} to provide meaningful privacy guarantees.

\textbf{Prediction/Inference privacy.} Only a handful of studies have addressed privacy of prediction by adding noise to the data. 
Osia et al.~\cite{osia1} employed dimensionality reduction techniques to reduce the amount of information before sending it to an untrusted cloud service. 
Wang et al.~\cite{notjust} propose a noise injection framework that randomly nullifies input elements for private inference, but their method requires retraining of the entire network. 
Leroux et al.~\cite{leroux2018privacy} propose an autoencoder to randomize the data, but the intensity of their obfuscation is too small to be irreversible, as they state.

Liu et al.~\cite{Liu2017DEEProtectEI} propose DEEProtect, an information-theoretic method which offers two usage modes for protecting privacy. One where it assumes no access to the privacy-sensitive inference labels and one where it assumes access to the privacy-sensitive labels.
Deeprotect incorporates the sensitive inference into its formulation for the latter usage mode.
%
%
%
A more recent work, Shredder~\cite{shredder}, proposes to \emph{heuristically} sample and reorder additive noise at run time based on the previously collected additive tensors that the DNN can tolerate (anti-adversarial patterns).
%
%
In contrast, \sieve's approach is to directly reduce information by learning conducive features and suppressing non-conducive ones with learned constant values. 
%
%
We also experimentally show that \sieve outperforms this prior work.
More importantly, this prior work relies on executing parts of the network on the edge side and sending the results to the cloud.
However, this separation is not always possible, as the service providers might not be willing to share the model parameters or change their infrastructure to accommodate for this method. Also, in some cases, the edge device might be incapable of running the first convolution layers of the neural network.
In contrast, we show that \sieve can perform equally efficiently in black-box settings without the collaboration of the service provider. 
%
%


Privacy on offloaded computation can also be provided by the means of cryptographic tools such as homomorphic encryption and/or Secure Multiparty Computation (SMC)~\cite{gazelle, dowlin16icml, minionn, delphi, falcon, feng2020cryptogru, lou2020autoprivacy, yalame}. 
%
%
However, these approaches suffer from a prohibitive computational cost (Table~\ref{table:crypto}), on both the cloud and user side, exacerbating the complexity and compute-intensity of neural networks especially on resource-constrained edge devices. 
\cloak, in contrast, avoids the significant cost of encryption and homomorphic data processing.

Several other research~\cite{tramer2018slalom,mlcapsule, murali} rely on trusted execution environments to remotely run ML algorithms. 
However, this model requires the users to send their data to an enclave running on remote servers and is vulnerable to the new breaches in hardware~\cite{spectre,meltdown}.
\vspace{-2ex}
\section{Conclusion}

The surge in the use of machine learning is driven by the growth in data and 
compute power. 
The data mostly comes from people~\cite{nytimes2} and includes an abundance of private information. 
We propose \cloak, a mechanism that finds features in the data that are unimportant and non-conducive for a cloud ML prediction model. 
This enables \sieve to suppress those features before sending them to the cloud, providing only the minimum information exposure
necessary to receive the particular service.
In doing so, \sieve not only minimizes the impact on the utility of the service, but it also imposes minimal overhead on the response time of the prediction service.

\vspace{-2ex}
\begin{acks}

We thank the anonymous reviewers for their insightful suggestions and comments.
This work was in part supported by National Science Foundation (NSF) awards CNS\#1703812, ECCS\#1609823, CCF\#1553192, CCF\#1823444,
Air Force Office of Scientific Research Young Investigator Program (YIP) award \#FA9550-17-1-0274,  
National Institute of Health (NIH) award \#R01EB028350, and
%
Air Force Research Laboratory (AFRL) and Defense Advanced Research Project Agency (DARPA) under agreement number \#FA8650-20-2-7009 and \#HR0011-18-C-0020.
The U.S. Government is authorized to reproduce and distribute reprints for Governmental purposes notwithstanding any copyright notation thereon.
The views and conclusions contained herein are those of the authors and should not be interpreted as necessarily representing the official policies or endorsements, either expressed or implied, of Samsung, Amazon, NSF, AFSOR, NIH, AFRL, DARPA, or the U.S. Government.
\end{acks}
\bibliographystyle{ACM-Reference-Format}
\bibliography{acmart}


\begin{thebibliography}{88}


\ifx \showCODEN    \undefined \def \showCODEN     #1{\unskip}     \fi
\ifx \showDOI      \undefined \def \showDOI       #1{#1}\fi
\ifx \showISBNx    \undefined \def \showISBNx     #1{\unskip}     \fi
\ifx \showISBNxiii \undefined \def \showISBNxiii  #1{\unskip}     \fi
\ifx \showISSN     \undefined \def \showISSN      #1{\unskip}     \fi
\ifx \showLCCN     \undefined \def \showLCCN      #1{\unskip}     \fi
\ifx \shownote     \undefined \def \shownote      #1{#1}          \fi
\ifx \showarticletitle \undefined \def \showarticletitle #1{#1}   \fi
\ifx \showURL      \undefined \def \showURL       {\relax}        \fi
\providecommand\bibfield[2]{#2}
\providecommand\bibinfo[2]{#2}
\providecommand\natexlab[1]{#1}
\providecommand\showeprint[2][]{arXiv:#2}

\bibitem[\protect\citeauthoryear{Abadi, Chu, Goodfellow, McMahan, Mironov,
  Talwar, and Zhang}{Abadi et~al\mbox{.}}{2016}]%
        {abadiDNN}
\bibfield{author}{\bibinfo{person}{Martin Abadi}, \bibinfo{person}{Andy Chu},
  \bibinfo{person}{Ian Goodfellow}, \bibinfo{person}{H.~Brendan McMahan},
  \bibinfo{person}{Ilya Mironov}, \bibinfo{person}{Kunal Talwar}, {and}
  \bibinfo{person}{Li Zhang}.} \bibinfo{year}{2016}\natexlab{}.
\newblock \showarticletitle{Deep Learning with Differential Privacy}. In
  \bibinfo{booktitle}{\emph{ACM Conference on Computer and Communications
  Security (CCS)}}.
\newblock


\bibitem[\protect\citeauthoryear{Agrawal, Shahin~Shamsabadi, Kusner, and
  Gasc{\'o}n}{Agrawal et~al\mbox{.}}{2019}]%
        {agrawal2019quotient}
\bibfield{author}{\bibinfo{person}{Nitin Agrawal}, \bibinfo{person}{Ali
  Shahin~Shamsabadi}, \bibinfo{person}{Matt~J Kusner}, {and}
  \bibinfo{person}{Adri{\`a} Gasc{\'o}n}.} \bibinfo{year}{2019}\natexlab{}.
\newblock \showarticletitle{QUOTIENT: two-party secure neural network training
  and prediction}. In \bibinfo{booktitle}{\emph{ACM Conference on Computer and
  Communications Security (CCS)}}.
\newblock


\bibitem[\protect\citeauthoryear{Arpit, Jastrz{k{e}}bski, Ballas, Krueger,
  Bengio, Kanwal, Maharaj, Fischer, Courville, Bengio, et~al\mbox{.}}{Arpit
  et~al\mbox{.}}{2017}]%
        {memr1}
\bibfield{author}{\bibinfo{person}{Devansh Arpit}, \bibinfo{person}{Stanis{l}aw
  Jastrz{k{e}}bski}, \bibinfo{person}{Nicolas Ballas}, \bibinfo{person}{David
  Krueger}, \bibinfo{person}{Emmanuel Bengio}, \bibinfo{person}{Maxinder~S
  Kanwal}, \bibinfo{person}{Tegan Maharaj}, \bibinfo{person}{Asja Fischer},
  \bibinfo{person}{Aaron Courville}, \bibinfo{person}{Yoshua Bengio},
  {et~al\mbox{.}}} \bibinfo{year}{2017}\natexlab{}.
\newblock \showarticletitle{A closer look at memorization in deep networks}. In
  \bibinfo{booktitle}{\emph{International Conference on Machine Learning
  (ICML)}}.
\newblock


\bibitem[\protect\citeauthoryear{Azad, Laperdrix, and Nikiforakis}{Azad
  et~al\mbox{.}}{2019}]%
        {more1}
\bibfield{author}{\bibinfo{person}{Babak~Amin Azad}, \bibinfo{person}{Pierre
  Laperdrix}, {and} \bibinfo{person}{Nick Nikiforakis}.}
  \bibinfo{year}{2019}\natexlab{}.
\newblock \showarticletitle{Less is more: quantifying the security benefits of
  debloating web applications}. In \bibinfo{booktitle}{\emph{28th
  $\{$USENIX$\}$ Security Symposium ($\{$USENIX$\}$ Security 19)}}.
  \bibinfo{pages}{1697--1714}.
\newblock


\bibitem[\protect\citeauthoryear{Balle, Kairouz, McMahan, Thakkar, and
  Thakurta}{Balle et~al\mbox{.}}{2020}]%
        {balle2020privacy}
\bibfield{author}{\bibinfo{person}{Borja Balle}, \bibinfo{person}{Peter
  Kairouz}, \bibinfo{person}{H~Brendan McMahan}, \bibinfo{person}{Om Thakkar},
  {and} \bibinfo{person}{Abhradeep Thakurta}.} \bibinfo{year}{2020}\natexlab{}.
\newblock \showarticletitle{Privacy amplification via random check-ins}.
\newblock \bibinfo{journal}{\emph{arXiv preprint arXiv:2007.06605}}
  (\bibinfo{year}{2020}).
\newblock


\bibitem[\protect\citeauthoryear{Barron, Miramirkhani, and Nikiforakis}{Barron
  et~al\mbox{.}}{2019}]%
        {del1}
\bibfield{author}{\bibinfo{person}{Timothy Barron}, \bibinfo{person}{Najmeh
  Miramirkhani}, {and} \bibinfo{person}{Nick Nikiforakis}.}
  \bibinfo{year}{2019}\natexlab{}.
\newblock \showarticletitle{Now You See It, Now You Don{\textquoteright}t: A
  Large-scale Analysis of Early Domain Deletions}. In
  \bibinfo{booktitle}{\emph{22nd International Symposium on Research in
  Attacks, Intrusions and Defenses ({RAID} 2019)}}.
  \bibinfo{publisher}{{USENIX} Association}, \bibinfo{address}{Chaoyang
  District, Beijing}, \bibinfo{pages}{383--397}.
\newblock
\showISBNx{978-1-939133-07-6}
\urldef\tempurl%
\url{https://www.usenix.org/conference/raid2019/presentation/barron}
\showURL{%
\tempurl}


\bibitem[\protect\citeauthoryear{Beaudry and Renner}{Beaudry and
  Renner}{2011}]%
        {norm2011intuitive}
\bibfield{author}{\bibinfo{person}{Normand~J. Beaudry} {and}
  \bibinfo{person}{Renato Renner}.} \bibinfo{year}{2011}\natexlab{}.
\newblock \showarticletitle{An intuitive proof of the data processing
  inequality}.
\newblock \bibinfo{journal}{\emph{arXiv preprint arXiv:1107.0740}}
  (\bibinfo{year}{2011}).
\newblock
\showeprint[arxiv]{1107.0740}~[quant-ph]


\bibitem[\protect\citeauthoryear{Blundell, Cornebise, Kavukcuoglu, and
  Wierstra}{Blundell et~al\mbox{.}}{2015}]%
        {BBB}
\bibfield{author}{\bibinfo{person}{Charles Blundell}, \bibinfo{person}{Julien
  Cornebise}, \bibinfo{person}{Koray Kavukcuoglu}, {and} \bibinfo{person}{Daan
  Wierstra}.} \bibinfo{year}{2015}\natexlab{}.
\newblock \showarticletitle{Weight uncertainty in neural networks}.
\newblock \bibinfo{journal}{\emph{arXiv preprint arXiv:1505.05424}}
  (\bibinfo{year}{2015}).
\newblock


\bibitem[\protect\citeauthoryear{Boemer, Cammarota, Demmler, Schneider, and
  Yalame}{Boemer et~al\mbox{.}}{2020}]%
        {yalame}
\bibfield{author}{\bibinfo{person}{Fabian Boemer}, \bibinfo{person}{Rosario
  Cammarota}, \bibinfo{person}{Daniel Demmler}, \bibinfo{person}{Thomas
  Schneider}, {and} \bibinfo{person}{Hossein Yalame}.}
  \bibinfo{year}{2020}\natexlab{}.
\newblock \showarticletitle{MP2ML: A Mixed-Protocol Machine Learning Framework
  for Private Inference}. In \bibinfo{booktitle}{\emph{Proceedings of the 15th
  International Conference on Availability, Reliability and Security}} (Virtual
  Event, Ireland) \emph{(\bibinfo{series}{ARES '20})}.
  \bibinfo{publisher}{Association for Computing Machinery},
  \bibinfo{address}{New York, NY, USA}, Article \bibinfo{articleno}{14},
  \bibinfo{numpages}{10}~pages.
\newblock
\showISBNx{9781450388337}
\urldef\tempurl%
\url{https://doi.org/10.1145/3407023.3407045}
\showDOI{\tempurl}


\bibitem[\protect\citeauthoryear{Bonawitz, Ivanov, Kreuter, Marcedone, McMahan,
  Patel, Ramage, Segal, and Seth}{Bonawitz et~al\mbox{.}}{2017}]%
        {fl-secureagg}
\bibfield{author}{\bibinfo{person}{Keith Bonawitz}, \bibinfo{person}{Vladimir
  Ivanov}, \bibinfo{person}{Ben Kreuter}, \bibinfo{person}{Antonio Marcedone},
  \bibinfo{person}{H~Brendan McMahan}, \bibinfo{person}{Sarvar Patel},
  \bibinfo{person}{Daniel Ramage}, \bibinfo{person}{Aaron Segal}, {and}
  \bibinfo{person}{Karn Seth}.} \bibinfo{year}{2017}\natexlab{}.
\newblock \showarticletitle{Practical secure aggregation for privacy-preserving
  machine learning}. In \bibinfo{booktitle}{\emph{proceedings of the 2017 ACM
  SIGSAC Conference on Computer and Communications Security}}.
  \bibinfo{pages}{1175--1191}.
\newblock


\bibitem[\protect\citeauthoryear{Carlini, Jagielski, and Mironov}{Carlini
  et~al\mbox{.}}{2020}]%
        {carlini2020cryptanalytic}
\bibfield{author}{\bibinfo{person}{Nicholas Carlini}, \bibinfo{person}{Matthew
  Jagielski}, {and} \bibinfo{person}{Ilya Mironov}.}
  \bibinfo{year}{2020}\natexlab{}.
\newblock \showarticletitle{Cryptanalytic extraction of neural network models}.
  In \bibinfo{booktitle}{\emph{Annual International Cryptology Conference}}.
  Springer, \bibinfo{pages}{189--218}.
\newblock


\bibitem[\protect\citeauthoryear{Chaudhuri and Monteleoni}{Chaudhuri and
  Monteleoni}{2009}]%
        {chaudhuriLogistic}
\bibfield{author}{\bibinfo{person}{Kamalika Chaudhuri} {and}
  \bibinfo{person}{Claire Monteleoni}.} \bibinfo{year}{2009}\natexlab{}.
\newblock \showarticletitle{Privacy-preserving logistic regression}.
\newblock In \bibinfo{booktitle}{\emph{Advances in Neural Information
  Processing Systems 21}}, \bibfield{editor}{\bibinfo{person}{D.~Koller},
  \bibinfo{person}{D.~Schuurmans}, \bibinfo{person}{Y.~Bengio}, {and}
  \bibinfo{person}{L.~Bottou}} (Eds.). \bibinfo{publisher}{Curran Associates,
  Inc.}, \bibinfo{pages}{289--296}.
\newblock
\urldef\tempurl%
\url{http://papers.nips.cc/paper/3486-privacy-preserving-logistic-regression.pdf}
\showURL{%
\tempurl}


\bibitem[\protect\citeauthoryear{Chaudhuri, Monteleoni, and Sarwate}{Chaudhuri
  et~al\mbox{.}}{2009}]%
        {chaudhuriERM}
\bibfield{author}{\bibinfo{person}{Kamalika Chaudhuri}, \bibinfo{person}{Claire
  Monteleoni}, {and} \bibinfo{person}{Anand~D. Sarwate}.}
  \bibinfo{year}{2009}\natexlab{}.
\newblock \showarticletitle{Differentially Private Empirical Risk
  Minimization}.
\newblock \bibinfo{journal}{\emph{arXiv preprint arXiv:0912.0071}}
  (\bibinfo{year}{2009}).
\newblock
\showeprint[arxiv]{0912.0071}~[cs.LG]


\bibitem[\protect\citeauthoryear{Chaudhuri, Sarwate, and Sinha}{Chaudhuri
  et~al\mbox{.}}{2013}]%
        {chaudhuriPCA}
\bibfield{author}{\bibinfo{person}{Kamalika Chaudhuri},
  \bibinfo{person}{Anand~D. Sarwate}, {and} \bibinfo{person}{Kaushik Sinha}.}
  \bibinfo{year}{2013}\natexlab{}.
\newblock \showarticletitle{A Near-Optimal Algorithm for Differentially-Private
  Principal Components}.
\newblock \bibinfo{journal}{\emph{J. Mach. Learn. Res.}} \bibinfo{volume}{14},
  \bibinfo{number}{1} (\bibinfo{date}{Jan.} \bibinfo{year}{2013}),
  \bibinfo{pages}{2905–2943}.
\newblock
\showISSN{1532-4435}


\bibitem[\protect\citeauthoryear{Cover and Thomas}{Cover and Thomas}{2012}]%
        {elements-book}
\bibfield{author}{\bibinfo{person}{Thomas~M Cover} {and} \bibinfo{person}{Joy~A
  Thomas}.} \bibinfo{year}{2012}\natexlab{}.
\newblock \bibinfo{booktitle}{\emph{Elements of information theory}}.
\newblock \bibinfo{publisher}{John Wiley \& Sons}.
\newblock


\bibitem[\protect\citeauthoryear{Cuff and Yu}{Cuff and Yu}{2016}]%
        {MI2016}
\bibfield{author}{\bibinfo{person}{Paul~W. Cuff} {and} \bibinfo{person}{Lanqing
  Yu}.} \bibinfo{year}{2016}\natexlab{}.
\newblock \showarticletitle{Differential Privacy as a Mutual Information
  Constraint}. In \bibinfo{booktitle}{\emph{ACM Conference on Computer and
  Communications Security (CCS)}}.
\newblock


\bibitem[\protect\citeauthoryear{Dong, Qiu, and Zhu}{Dong
  et~al\mbox{.}}{2014}]%
        {more2}
\bibfield{author}{\bibinfo{person}{Wei Dong}, \bibinfo{person}{Minghui Qiu},
  {and} \bibinfo{person}{Feida Zhu}.} \bibinfo{year}{2014}\natexlab{}.
\newblock \showarticletitle{Who am I on twitter? a cross-country comparison}.
  In \bibinfo{booktitle}{\emph{Proceedings of the 23rd International Conference
  on World Wide Web}}. \bibinfo{pages}{253--254}.
\newblock


\bibitem[\protect\citeauthoryear{Dowlin, Gilad-Bachrach, Laine, Lauter,
  Naehrig, and Wernsing}{Dowlin et~al\mbox{.}}{2016}]%
        {dowlin16icml}
\bibfield{author}{\bibinfo{person}{Nathan Dowlin}, \bibinfo{person}{Ran
  Gilad-Bachrach}, \bibinfo{person}{Kim Laine}, \bibinfo{person}{Kristin
  Lauter}, \bibinfo{person}{Michael Naehrig}, {and} \bibinfo{person}{John
  Wernsing}.} \bibinfo{year}{2016}\natexlab{}.
\newblock \showarticletitle{CryptoNets: Applying Neural Networks to Encrypted
  Data with High Throughput and Accuracy}. In
  \bibinfo{booktitle}{\emph{International Conference on Machine Learning
  (ICML)}}.
\newblock


\bibitem[\protect\citeauthoryear{Dwork, Kenthapadi, McSherry, Mironov, and
  Naor}{Dwork et~al\mbox{.}}{2006a}]%
        {dwork06euro}
\bibfield{author}{\bibinfo{person}{Cynthia Dwork}, \bibinfo{person}{Krishnaram
  Kenthapadi}, \bibinfo{person}{Frank McSherry}, \bibinfo{person}{Ilya
  Mironov}, {and} \bibinfo{person}{Moni Naor}.}
  \bibinfo{year}{2006}\natexlab{a}.
\newblock \showarticletitle{Our Data, Ourselves: Privacy via Distributed Noise
  Generation}. In \bibinfo{booktitle}{\emph{Proceedings of the 24th Annual
  International Conference on The Theory and Applications of Cryptographic
  Techniques}} (St. Petersburg, Russia)
  \emph{(\bibinfo{series}{EUROCRYPT'06})}.
  \bibinfo{publisher}{Springer-Verlag}, \bibinfo{address}{Berlin, Heidelberg},
  \bibinfo{pages}{486--503}.
\newblock
\showISBNx{3-540-34546-9, 978-3-540-34546-6}
\urldef\tempurl%
\url{https://doi.org/10.1007/11761679_29}
\showDOI{\tempurl}


\bibitem[\protect\citeauthoryear{Dwork, McSherry, Nissim, and Smith}{Dwork
  et~al\mbox{.}}{2006b}]%
        {dwork06tcc}
\bibfield{author}{\bibinfo{person}{Cynthia Dwork}, \bibinfo{person}{Frank
  McSherry}, \bibinfo{person}{Kobbi Nissim}, {and} \bibinfo{person}{Adam
  Smith}.} \bibinfo{year}{2006}\natexlab{b}.
\newblock \showarticletitle{Calibrating Noise to Sensitivity in Private Data
  Analysis}. In \bibinfo{booktitle}{\emph{Proceedings of the Third Conference
  on Theory of Cryptography}} (New York, NY) \emph{(\bibinfo{series}{TCC'06})}.
  \bibinfo{publisher}{Springer-Verlag}, \bibinfo{address}{Berlin, Heidelberg},
  \bibinfo{pages}{265--284}.
\newblock
\showISBNx{3-540-32731-2, 978-3-540-32731-8}
\urldef\tempurl%
\url{https://doi.org/10.1007/11681878_14}
\showDOI{\tempurl}


\bibitem[\protect\citeauthoryear{Dwork and Roth}{Dwork and Roth}{2014}]%
        {dwork14book}
\bibfield{author}{\bibinfo{person}{Cynthia Dwork} {and} \bibinfo{person}{Aaron
  Roth}.} \bibinfo{year}{2014}\natexlab{}.
\newblock \showarticletitle{The Algorithmic Foundations of Differential
  Privacy}.
\newblock \bibinfo{journal}{\emph{Found. Trends Theor. Comput. Sci.}}
  \bibinfo{volume}{9} (\bibinfo{date}{Aug.} \bibinfo{year}{2014}),
  \bibinfo{pages}{211--407}.
\newblock
\showISSN{1551-305X}
\urldef\tempurl%
\url{https://doi.org/10.1561/0400000042}
\showDOI{\tempurl}


\bibitem[\protect\citeauthoryear{Facebook}{Facebook}{2019}]%
        {crypten}
\bibfield{author}{\bibinfo{person}{Facebook}.} \bibinfo{year}{2019}\natexlab{}.
\newblock \bibinfo{title}{A research tool for secure machine learning in
  PyTorch}.
\newblock
\newblock
\newblock
\shownote{online--accessed June 2020, url: \url{ https://crypten.ai}.}


\bibitem[\protect\citeauthoryear{Feng, Lou, Jiang, and Fox}{Feng
  et~al\mbox{.}}{2020}]%
        {feng2020cryptogru}
\bibfield{author}{\bibinfo{person}{Bo Feng}, \bibinfo{person}{Qian Lou},
  \bibinfo{person}{Lei Jiang}, {and} \bibinfo{person}{Geoffrey~C Fox}.}
  \bibinfo{year}{2020}\natexlab{}.
\newblock \showarticletitle{CryptoGRU: Low Latency Privacy-Preserving Text
  Analysis With GRU}.
\newblock \bibinfo{journal}{\emph{arXiv preprint arXiv:2010.11796}}
  (\bibinfo{year}{2020}).
\newblock


\bibitem[\protect\citeauthoryear{Hanzlik, Zhang, Grosse, Salem, Augustin,
  Backes, and Fritz}{Hanzlik et~al\mbox{.}}{2018}]%
        {mlcapsule}
\bibfield{author}{\bibinfo{person}{Lucjan Hanzlik}, \bibinfo{person}{Yang
  Zhang}, \bibinfo{person}{Kathrin Grosse}, \bibinfo{person}{Ahmed Salem},
  \bibinfo{person}{Max Augustin}, \bibinfo{person}{Michael Backes}, {and}
  \bibinfo{person}{Mario Fritz}.} \bibinfo{year}{2018}\natexlab{}.
\newblock \showarticletitle{MLCapsule: Guarded Offline Deployment of Machine
  Learning as a Service}.
\newblock \bibinfo{journal}{\emph{arXiv preprint arXiv:1808.00590}}
  (\bibinfo{year}{2018}).
\newblock
\showeprint[arxiv]{1808.00590}~[cs.CR]


\bibitem[\protect\citeauthoryear{Hashemi, Wang, and Annavaram}{Hashemi
  et~al\mbox{.}}{2020}]%
        {hashemi2020darknight}
\bibfield{author}{\bibinfo{person}{Hanieh Hashemi}, \bibinfo{person}{Yongqin
  Wang}, {and} \bibinfo{person}{Murali Annavaram}.}
  \bibinfo{year}{2020}\natexlab{}.
\newblock \showarticletitle{DarKnight: A Data Privacy Scheme for Training and
  Inference of Deep Neural Networks}.
\newblock \bibinfo{journal}{\emph{arXiv preprint arXiv:2006.01300}}
  (\bibinfo{year}{2020}).
\newblock


\bibitem[\protect\citeauthoryear{He, Zhang, Ren, and Sun}{He
  et~al\mbox{.}}{2016}]%
        {resnet}
\bibfield{author}{\bibinfo{person}{Kaiming He}, \bibinfo{person}{Xiangyu
  Zhang}, \bibinfo{person}{Shaoqing Ren}, {and} \bibinfo{person}{Jian Sun}.}
  \bibinfo{year}{2016}\natexlab{}.
\newblock \showarticletitle{Deep residual learning for image recognition}. In
  \bibinfo{booktitle}{\emph{IEEE Conference on Computer Vision and Pattern
  Recognition (CVPR)}}.
\newblock


\bibitem[\protect\citeauthoryear{He, Zhang, and Lee}{He et~al\mbox{.}}{2019}]%
        {he2019model}
\bibfield{author}{\bibinfo{person}{Zecheng He}, \bibinfo{person}{Tianwei
  Zhang}, {and} \bibinfo{person}{Ruby~B Lee}.} \bibinfo{year}{2019}\natexlab{}.
\newblock \showarticletitle{Model inversion attacks against collaborative
  inference}. In \bibinfo{booktitle}{\emph{Proceedings of the 35th Annual
  Computer Security Applications Conference}}. \bibinfo{pages}{148--162}.
\newblock


\bibitem[\protect\citeauthoryear{Jia, Salem, Backes, Zhang, and Gong}{Jia
  et~al\mbox{.}}{2019}]%
        {memgaurd}
\bibfield{author}{\bibinfo{person}{Jinyuan Jia}, \bibinfo{person}{Ahmed Salem},
  \bibinfo{person}{Michael Backes}, \bibinfo{person}{Yang Zhang}, {and}
  \bibinfo{person}{Neil~Zhenqiang Gong}.} \bibinfo{year}{2019}\natexlab{}.
\newblock \showarticletitle{MemGuard: Defending against Black-Box Membership
  Inference Attacks via Adversarial Examples}. In \bibinfo{booktitle}{\emph{ACM
  Conference on Computer and Communications Security (CCS)}}.
\newblock


\bibitem[\protect\citeauthoryear{Jiang, Ji, Wang, Mohammed, Cheng, and
  Ohno-Machado}{Jiang et~al\mbox{.}}{2013}]%
        {DPPCA}
\bibfield{author}{\bibinfo{person}{Xiaoqian Jiang}, \bibinfo{person}{Zhanglong
  Ji}, \bibinfo{person}{Shuang Wang}, \bibinfo{person}{Noman Mohammed},
  \bibinfo{person}{Samuel Cheng}, {and} \bibinfo{person}{Lucila Ohno-Machado}.}
  \bibinfo{year}{2013}\natexlab{}.
\newblock \showarticletitle{Differential-private data publishing through
  component analysis}.
\newblock \bibinfo{journal}{\emph{Transactions on data privacy}}
  \bibinfo{volume}{6}, \bibinfo{number}{1} (\bibinfo{year}{2013}),
  \bibinfo{pages}{19}.
\newblock


\bibitem[\protect\citeauthoryear{Juvekar, Vaikuntanathan, and
  Chandrakasan}{Juvekar et~al\mbox{.}}{2018}]%
        {gazelle}
\bibfield{author}{\bibinfo{person}{Chiraag Juvekar}, \bibinfo{person}{Vinod
  Vaikuntanathan}, {and} \bibinfo{person}{Anantha Chandrakasan}.}
  \bibinfo{year}{2018}\natexlab{}.
\newblock \showarticletitle{{GAZELLE}: A Low Latency Framework for Secure
  Neural Network Inference}. In \bibinfo{booktitle}{\emph{{USENIX} Security
  Symposium ({USENIX} Security)}}.
\newblock


\bibitem[\protect\citeauthoryear{Kairouz, Liao, Huang, and Sankar}{Kairouz
  et~al\mbox{.}}{2019a}]%
        {kairouz2019censored}
\bibfield{author}{\bibinfo{person}{Peter Kairouz}, \bibinfo{person}{Jiachun
  Liao}, \bibinfo{person}{Chong Huang}, {and} \bibinfo{person}{Lalitha
  Sankar}.} \bibinfo{year}{2019}\natexlab{a}.
\newblock \showarticletitle{Censored and Fair Universal Representations using
  Generative Adversarial Models}.
\newblock \bibinfo{journal}{\emph{arXiv preprint arXiv:1910.00411}}
  (\bibinfo{year}{2019}).
\newblock
\showeprint[arxiv]{1910.00411}~[cs.LG]


\bibitem[\protect\citeauthoryear{Kairouz, McMahan, Avent, Bellet, Bennis,
  Bhagoji, Bonawitz, Charles, Cormode, Cummings, et~al\mbox{.}}{Kairouz
  et~al\mbox{.}}{2019b}]%
        {flsurvey}
\bibfield{author}{\bibinfo{person}{Peter Kairouz}, \bibinfo{person}{H~Brendan
  McMahan}, \bibinfo{person}{Brendan Avent}, \bibinfo{person}{Aur{\'e}lien
  Bellet}, \bibinfo{person}{Mehdi Bennis}, \bibinfo{person}{Arjun~Nitin
  Bhagoji}, \bibinfo{person}{Keith Bonawitz}, \bibinfo{person}{Zachary
  Charles}, \bibinfo{person}{Graham Cormode}, \bibinfo{person}{Rachel
  Cummings}, {et~al\mbox{.}}} \bibinfo{year}{2019}\natexlab{b}.
\newblock \showarticletitle{Advances and open problems in federated learning}.
\newblock \bibinfo{journal}{\emph{arXiv preprint arXiv:1912.04977}}
  (\bibinfo{year}{2019}).
\newblock


\bibitem[\protect\citeauthoryear{{Kalantari}, {Sankar}, and
  {Kosut}}{{Kalantari} et~al\mbox{.}}{2017}]%
        {lalitha17}
\bibfield{author}{\bibinfo{person}{K. {Kalantari}}, \bibinfo{person}{L.
  {Sankar}}, {and} \bibinfo{person}{O. {Kosut}}.}
  \bibinfo{year}{2017}\natexlab{}.
\newblock \showarticletitle{On information-theoretic privacy with general
  distortion cost functions}. In \bibinfo{booktitle}{\emph{2017 IEEE
  International Symposium on Information Theory (ISIT)}}.
  \bibinfo{pages}{2865--2869}.
\newblock
\showISSN{2157-8117}
\urldef\tempurl%
\url{https://doi.org/10.1109/ISIT.2017.8007053}
\showDOI{\tempurl}


\bibitem[\protect\citeauthoryear{Kalos and Whitlock}{Kalos and
  Whitlock}{1986}]%
        {monte}
\bibfield{author}{\bibinfo{person}{Malvin~H. Kalos} {and}
  \bibinfo{person}{Paula~A. Whitlock}.} \bibinfo{year}{1986}\natexlab{}.
\newblock \bibinfo{booktitle}{\emph{Monte Carlo Methods. Vol. 1: Basics}}.
\newblock \bibinfo{publisher}{Wiley-Interscience}, \bibinfo{address}{USA}.
\newblock
\showISBNx{0471898392}


\bibitem[\protect\citeauthoryear{Kocher, Horn, Fogh, , Genkin, Gruss, Haas,
  Hamburg, Lipp, Mangard, Prescher, Schwarz, and Yarom}{Kocher
  et~al\mbox{.}}{2019}]%
        {spectre}
\bibfield{author}{\bibinfo{person}{Paul Kocher}, \bibinfo{person}{Jann Horn},
  \bibinfo{person}{Anders Fogh}, \bibinfo{person}{}, \bibinfo{person}{Daniel
  Genkin}, \bibinfo{person}{Daniel Gruss}, \bibinfo{person}{Werner Haas},
  \bibinfo{person}{Mike Hamburg}, \bibinfo{person}{Moritz Lipp},
  \bibinfo{person}{Stefan Mangard}, \bibinfo{person}{Thomas Prescher},
  \bibinfo{person}{Michael Schwarz}, {and} \bibinfo{person}{Yuval Yarom}.}
  \bibinfo{year}{2019}\natexlab{}.
\newblock \showarticletitle{Spectre Attacks: Exploiting Speculative Execution}.
  In \bibinfo{booktitle}{\emph{IEEE Symposium on Security and Privacy (S\&P)}}.
\newblock


\bibitem[\protect\citeauthoryear{Krishna, Tomar, Parikh, Papernot, and
  Iyyer}{Krishna et~al\mbox{.}}{2019}]%
        {krishna2019thieves}
\bibfield{author}{\bibinfo{person}{Kalpesh Krishna},
  \bibinfo{person}{Gaurav~Singh Tomar}, \bibinfo{person}{Ankur~P Parikh},
  \bibinfo{person}{Nicolas Papernot}, {and} \bibinfo{person}{Mohit Iyyer}.}
  \bibinfo{year}{2019}\natexlab{}.
\newblock \showarticletitle{Thieves on sesame street! model extraction of
  bert-based apis}.
\newblock \bibinfo{journal}{\emph{arXiv preprint arXiv:1910.12366}}
  (\bibinfo{year}{2019}).
\newblock


\bibitem[\protect\citeauthoryear{Krizhevsky, Nair, and Hinton}{Krizhevsky
  et~al\mbox{.}}{[n.d.]}]%
        {cifar100}
\bibfield{author}{\bibinfo{person}{Alex Krizhevsky}, \bibinfo{person}{Vinod
  Nair}, {and} \bibinfo{person}{Geoffrey Hinton}.}
  \bibinfo{year}{[n.d.]}\natexlab{}.
\newblock \showarticletitle{CIFAR-100 (Canadian Institute for Advanced
  Research)}.
\newblock  (\bibinfo{year}{[n.\,d.]}).
\newblock
\urldef\tempurl%
\url{http://www.cs.toronto.edu/~kriz/cifar.html}
\showURL{%
\tempurl}
\newblock
\shownote{url: \url{http://www.cs.toronto.edu/~kriz/cifar.html}.}


\bibitem[\protect\citeauthoryear{Krizhevsky, Sutskever, and Hinton}{Krizhevsky
  et~al\mbox{.}}{2012}]%
        {alexnet}
\bibfield{author}{\bibinfo{person}{Alex Krizhevsky}, \bibinfo{person}{Ilya
  Sutskever}, {and} \bibinfo{person}{Geoffrey~E. Hinton}.}
  \bibinfo{year}{2012}\natexlab{}.
\newblock \showarticletitle{ImageNet Classification with Deep Convolutional
  Neural Networks}.
\newblock \bibinfo{journal}{\emph{Commun. ACM}}  \bibinfo{volume}{60}
  (\bibinfo{year}{2012}), \bibinfo{pages}{84--90}.
\newblock


\bibitem[\protect\citeauthoryear{LeCun}{LeCun}{1998}]%
        {lenet5}
\bibfield{author}{\bibinfo{person}{Yann LeCun}.}
  \bibinfo{year}{1998}\natexlab{}.
\newblock \showarticletitle{Gradient-based learning applied to document
  recognition}.
\newblock


\bibitem[\protect\citeauthoryear{LeCun and Cortes}{LeCun and Cortes}{[n.d.]}]%
        {mnist}
\bibfield{author}{\bibinfo{person}{Yann LeCun} {and} \bibinfo{person}{Corinna
  Cortes}.} \bibinfo{year}{[n.d.]}\natexlab{}.
\newblock \bibinfo{title}{The MNIST Dataset Of Handwritten Digits}.
\newblock
\newblock
\newblock
\shownote{online accessed May 2019 \url{
  http://www.pymvpa.org/datadb/mnist.html}.}


\bibitem[\protect\citeauthoryear{Leroux, Verbelen, Simoens, and Dhoedt}{Leroux
  et~al\mbox{.}}{2018}]%
        {leroux2018privacy}
\bibfield{author}{\bibinfo{person}{Sam Leroux}, \bibinfo{person}{Tim Verbelen},
  \bibinfo{person}{Pieter Simoens}, {and} \bibinfo{person}{Bart Dhoedt}.}
  \bibinfo{year}{2018}\natexlab{}.
\newblock \bibinfo{title}{Privacy Aware Offloading of Deep Neural Networks}.
\newblock
\newblock
\showeprint[arxiv]{1805.12024}~[cs.LG]


\bibitem[\protect\citeauthoryear{Liao, Kosut, Sankar, and Calmon}{Liao
  et~al\mbox{.}}{2017}]%
        {lalitha17b}
\bibfield{author}{\bibinfo{person}{Jiachun Liao}, \bibinfo{person}{Oliver
  Kosut}, \bibinfo{person}{Lalitha Sankar}, {and}
  \bibinfo{person}{Fl{\'a}vio~P. Calmon}.} \bibinfo{year}{2017}\natexlab{}.
\newblock \bibinfo{title}{A General Framework for Information Leakage}.
\newblock
\newblock


\bibitem[\protect\citeauthoryear{Lipp, Schwarz, Gruss, Prescher, Haas, Fogh,
  Horn, Mangard, Kocher, Genkin, Yarom, and Hamburg}{Lipp
  et~al\mbox{.}}{2018}]%
        {meltdown}
\bibfield{author}{\bibinfo{person}{Moritz Lipp}, \bibinfo{person}{Michael
  Schwarz}, \bibinfo{person}{Daniel Gruss}, \bibinfo{person}{Thomas Prescher},
  \bibinfo{person}{Werner Haas}, \bibinfo{person}{Anders Fogh},
  \bibinfo{person}{Jann Horn}, \bibinfo{person}{Stefan Mangard},
  \bibinfo{person}{Paul Kocher}, \bibinfo{person}{Daniel Genkin},
  \bibinfo{person}{Yuval Yarom}, {and} \bibinfo{person}{Mike Hamburg}.}
  \bibinfo{year}{2018}\natexlab{}.
\newblock \showarticletitle{Meltdown: Reading Kernel Memory from User Space}.
  In \bibinfo{booktitle}{\emph{{USENIX} Security Symposium ({USENIX}
  Security)}}.
\newblock


\bibitem[\protect\citeauthoryear{Liu, Chakraborty, and Mittal}{Liu
  et~al\mbox{.}}{2017a}]%
        {Liu2017DEEProtectEI}
\bibfield{author}{\bibinfo{person}{C. Liu}, \bibinfo{person}{S. Chakraborty},
  {and} \bibinfo{person}{P. Mittal}.} \bibinfo{year}{2017}\natexlab{a}.
\newblock \showarticletitle{DEEProtect: Enabling Inference-based Access Control
  on Mobile Sensing Applications}.
\newblock \bibinfo{journal}{\emph{ArXiv}}  \bibinfo{volume}{abs/1702.06159}
  (\bibinfo{year}{2017}).
\newblock


\bibitem[\protect\citeauthoryear{Liu, Juuti, Lu, and Asokan}{Liu
  et~al\mbox{.}}{2017b}]%
        {minionn}
\bibfield{author}{\bibinfo{person}{Jian Liu}, \bibinfo{person}{Mika Juuti},
  \bibinfo{person}{Yao Lu}, {and} \bibinfo{person}{Nadarajah Asokan}.}
  \bibinfo{year}{2017}\natexlab{b}.
\newblock \showarticletitle{Oblivious neural network predictions via minionn
  transformations}. In \bibinfo{booktitle}{\emph{ACM Conference on Computer and
  Communications Security (CCS)}}.
\newblock


\bibitem[\protect\citeauthoryear{Liu, Li, Smith, and Sekar}{Liu
  et~al\mbox{.}}{2019}]%
        {liu2019enhancing}
\bibfield{author}{\bibinfo{person}{Zaoxing Liu}, \bibinfo{person}{Tian Li},
  \bibinfo{person}{Virginia Smith}, {and} \bibinfo{person}{Vyas Sekar}.}
  \bibinfo{year}{2019}\natexlab{}.
\newblock \showarticletitle{Enhancing the privacy of federated learning with
  sketching}.
\newblock \bibinfo{journal}{\emph{arXiv preprint arXiv:1911.01812}}
  (\bibinfo{year}{2019}).
\newblock


\bibitem[\protect\citeauthoryear{Liu, Luo, Wang, and Tang}{Liu
  et~al\mbox{.}}{2015}]%
        {celeba}
\bibfield{author}{\bibinfo{person}{Ziwei Liu}, \bibinfo{person}{Ping Luo},
  \bibinfo{person}{Xiaogang Wang}, {and} \bibinfo{person}{Xiaoou Tang}.}
  \bibinfo{year}{2015}\natexlab{}.
\newblock \showarticletitle{Deep Learning Face Attributes in the Wild}. In
  \bibinfo{booktitle}{\emph{International Conference on Computer Vision
  (ICCV)}}.
\newblock


\bibitem[\protect\citeauthoryear{Lou, Song, and Jiang}{Lou
  et~al\mbox{.}}{2020}]%
        {lou2020autoprivacy}
\bibfield{author}{\bibinfo{person}{Qian Lou}, \bibinfo{person}{Bian Song},
  {and} \bibinfo{person}{Lei Jiang}.} \bibinfo{year}{2020}\natexlab{}.
\newblock \showarticletitle{AutoPrivacy: Automated Layer-wise Parameter
  Selection for Secure Neural Network Inference}.
\newblock \bibinfo{journal}{\emph{arXiv preprint arXiv:2006.04219}}
  (\bibinfo{year}{2020}).
\newblock


\bibitem[\protect\citeauthoryear{Lu, Issaranon, and Forsyth}{Lu
  et~al\mbox{.}}{2017}]%
        {lu2017safetynet}
\bibfield{author}{\bibinfo{person}{Jiajun Lu}, \bibinfo{person}{Theerasit
  Issaranon}, {and} \bibinfo{person}{David Forsyth}.}
  \bibinfo{year}{2017}\natexlab{}.
\newblock \showarticletitle{Safetynet: Detecting and rejecting adversarial
  examples robustly}. In \bibinfo{booktitle}{\emph{International Conference on
  Computer Vision (ICCV)}}.
\newblock


\bibitem[\protect\citeauthoryear{Madras, Creager, Pitassi, and Zemel}{Madras
  et~al\mbox{.}}{2018}]%
        {madras}
\bibfield{author}{\bibinfo{person}{David Madras}, \bibinfo{person}{Elliot
  Creager}, \bibinfo{person}{Toniann Pitassi}, {and} \bibinfo{person}{Richard
  Zemel}.} \bibinfo{year}{2018}\natexlab{}.
\newblock \showarticletitle{Learning adversarially fair and transferable
  representations}.
\newblock \bibinfo{journal}{\emph{arXiv preprint arXiv:1802.06309}}
  (\bibinfo{year}{2018}).
\newblock


\bibitem[\protect\citeauthoryear{Mannan and van Oorschot}{Mannan and van
  Oorschot}{2008}]%
        {2008priv}
\bibfield{author}{\bibinfo{person}{Mohammad Mannan} {and}
  \bibinfo{person}{Paul~C. van Oorschot}.} \bibinfo{year}{2008}\natexlab{}.
\newblock \showarticletitle{Privacy-Enhanced Sharing of Personal Content on the
  Web}. In \bibinfo{booktitle}{\emph{Proceedings of the 17th International
  Conference on World Wide Web}} (Beijing, China) \emph{(\bibinfo{series}{WWW
  '08})}. \bibinfo{publisher}{Association for Computing Machinery},
  \bibinfo{address}{New York, NY, USA}, \bibinfo{pages}{487–496}.
\newblock
\showISBNx{9781605580852}
\urldef\tempurl%
\url{https://doi.org/10.1145/1367497.1367564}
\showDOI{\tempurl}


\bibitem[\protect\citeauthoryear{Mireshghallah, Taram, Ramrakhyani, Jalali,
  Tullsen, and Esmaeilzadeh}{Mireshghallah et~al\mbox{.}}{2020a}]%
        {shredder}
\bibfield{author}{\bibinfo{person}{Fatemehsadat Mireshghallah},
  \bibinfo{person}{Mohammadkazem Taram}, \bibinfo{person}{Prakash Ramrakhyani},
  \bibinfo{person}{Ali Jalali}, \bibinfo{person}{Dean Tullsen}, {and}
  \bibinfo{person}{Hadi Esmaeilzadeh}.} \bibinfo{year}{2020}\natexlab{a}.
\newblock \showarticletitle{Shredder: Learning Noise Distributions to Protect
  Inference Privacy}. In \bibinfo{booktitle}{\emph{International Conference on
  Architectural Support for Programming Languages and Operating Systems
  (ASPLOS)}}.
\newblock


\bibitem[\protect\citeauthoryear{Mireshghallah, Taram, Vepakomma, Singh,
  Raskar, and Esmaeilzadeh}{Mireshghallah et~al\mbox{.}}{2020b}]%
        {survey}
\bibfield{author}{\bibinfo{person}{Fatemehsadat Mireshghallah},
  \bibinfo{person}{Mohammadkazem Taram}, \bibinfo{person}{Praneeth Vepakomma},
  \bibinfo{person}{Abhishek Singh}, \bibinfo{person}{Ramesh Raskar}, {and}
  \bibinfo{person}{Hadi Esmaeilzadeh}.} \bibinfo{year}{2020}\natexlab{b}.
\newblock \showarticletitle{Privacy in Deep Learning: A Survey}. In
  \bibinfo{booktitle}{\emph{ArXiv}}, Vol.~\bibinfo{volume}{abs/2004.12254}.
\newblock


\bibitem[\protect\citeauthoryear{Mishra, Lehmkuhl, Srinivasan, Zheng, and
  Popa}{Mishra et~al\mbox{.}}{2020}]%
        {delphi}
\bibfield{author}{\bibinfo{person}{Pratyush Mishra}, \bibinfo{person}{Ryan
  Lehmkuhl}, \bibinfo{person}{Akshayaram Srinivasan}, \bibinfo{person}{Wenting
  Zheng}, {and} \bibinfo{person}{Raluca~Ada Popa}.}
  \bibinfo{year}{2020}\natexlab{}.
\newblock \showarticletitle{Delphi: A Cryptographic Inference Service for
  Neural Networks}. In \bibinfo{booktitle}{\emph{{USENIX} Security Symposium
  ({USENIX} Security)}}.
\newblock
\urldef\tempurl%
\url{https://www.usenix.org/conference/usenixsecurity20/presentation/mishra}
\showURL{%
\tempurl}


\bibitem[\protect\citeauthoryear{{MLPerf Organization}}{{MLPerf
  Organization}}{2020}]%
        {mlperf}
\bibfield{author}{\bibinfo{person}{{MLPerf Organization}}.}
  \bibinfo{year}{2020}\natexlab{}.
\newblock \bibinfo{title}{{MLPerf Benchmark Suite}}.
\newblock
\newblock
\newblock
\shownote{url: \url{https://mlperf.org}.}


\bibitem[\protect\citeauthoryear{{Mohassel} and {Zhang}}{{Mohassel} and
  {Zhang}}{2017}]%
        {secureml}
\bibfield{author}{\bibinfo{person}{P. {Mohassel}} {and} \bibinfo{person}{Y.
  {Zhang}}.} \bibinfo{year}{2017}\natexlab{}.
\newblock \showarticletitle{SecureML: A System for Scalable Privacy-Preserving
  Machine Learning}. In \bibinfo{booktitle}{\emph{IEEE Symposium on Security
  and Privacy (S\&P)}}.
\newblock


\bibitem[\protect\citeauthoryear{Mondal, Messias, Ghosh, Gummadi, and
  Kate}{Mondal et~al\mbox{.}}{2016}]%
        {soups}
\bibfield{author}{\bibinfo{person}{Mainack Mondal}, \bibinfo{person}{Johnnatan
  Messias}, \bibinfo{person}{Saptarshi Ghosh}, \bibinfo{person}{Krishna~P
  Gummadi}, {and} \bibinfo{person}{Aniket Kate}.}
  \bibinfo{year}{2016}\natexlab{}.
\newblock \showarticletitle{Forgetting in social media: Understanding and
  controlling longitudinal exposure of socially shared data}. In
  \bibinfo{booktitle}{\emph{Twelfth Symposium on Usable Privacy and Security
  ($\{$SOUPS$\}$ 2016)}}. \bibinfo{pages}{287--299}.
\newblock


\bibitem[\protect\citeauthoryear{Narra, Lin, Wang, Balasubramaniam, and
  Annavaram}{Narra et~al\mbox{.}}{2019}]%
        {murali}
\bibfield{author}{\bibinfo{person}{Krishna~Giri Narra},
  \bibinfo{person}{Zhifeng Lin}, \bibinfo{person}{Yongqin Wang},
  \bibinfo{person}{Keshav Balasubramaniam}, {and} \bibinfo{person}{Murali
  Annavaram}.} \bibinfo{year}{2019}\natexlab{}.
\newblock \showarticletitle{Privacy-Preserving Inference in Machine Learning
  Services Using Trusted Execution Environments}.
\newblock \bibinfo{journal}{\emph{arXiv preprint arXiv:1912.03485}}
  (\bibinfo{year}{2019}).
\newblock


\bibitem[\protect\citeauthoryear{Newcomb}{Newcomb}{2018}]%
        {facebook}
\bibfield{author}{\bibinfo{person}{Alyssa Newcomb}.}
  \bibinfo{year}{2018}\natexlab{}.
\newblock \bibinfo{title}{Facebook data harvesting scandal widens to 87 million
  people}.
\newblock
\newblock
\newblock
\shownote{online--accessed February 2020, url:\url{
  https://www.nbcnews.com/tech/tech-news/facebook-data-harvesting-scandal-widens-87-million-people-n862771}.}


\bibitem[\protect\citeauthoryear{{Osia}, {Shamsabadi}, {Sajadmanesh}, {Taheri},
  {Katevas}, {Rabiee}, {Lane}, and {Haddadi}}{{Osia} et~al\mbox{.}}{2020}]%
        {osia1}
\bibfield{author}{\bibinfo{person}{S.~A. {Osia}}, \bibinfo{person}{A.~S.
  {Shamsabadi}}, \bibinfo{person}{S. {Sajadmanesh}}, \bibinfo{person}{A.
  {Taheri}}, \bibinfo{person}{K. {Katevas}}, \bibinfo{person}{H.~R. {Rabiee}},
  \bibinfo{person}{N.~D. {Lane}}, {and} \bibinfo{person}{H. {Haddadi}}.}
  \bibinfo{year}{2020}\natexlab{}.
\newblock \showarticletitle{A Hybrid Deep Learning Architecture for
  Privacy-Preserving Mobile Analytics}.
\newblock \bibinfo{journal}{\emph{IEEE Internet of Things Journal}}
  (\bibinfo{year}{2020}), \bibinfo{pages}{1--1}.
\newblock
\showISSN{2372-2541}
\urldef\tempurl%
\url{https://doi.org/10.1109/JIOT.2020.2967734}
\showDOI{\tempurl}


\bibitem[\protect\citeauthoryear{Osia, Taheri, Shamsabadi, Katevas, Haddadi,
  and Rabiee}{Osia et~al\mbox{.}}{2020}]%
        {osia2}
\bibfield{author}{\bibinfo{person}{Seyed~Ali Osia}, \bibinfo{person}{Ali
  Taheri}, \bibinfo{person}{Ali~Shahin Shamsabadi}, \bibinfo{person}{Kleomenis
  Katevas}, \bibinfo{person}{Hamed Haddadi}, {and} \bibinfo{person}{Hamid~R.
  Rabiee}.} \bibinfo{year}{2020}\natexlab{}.
\newblock \showarticletitle{Deep Private-Feature Extraction}.
\newblock \bibinfo{journal}{\emph{IEEE Transactions on Knowledge and Data
  Engineering}} \bibinfo{volume}{32}, \bibinfo{number}{1} (\bibinfo{date}{Jan}
  \bibinfo{year}{2020}), \bibinfo{pages}{54–66}.
\newblock
\showISSN{2326-3865}
\urldef\tempurl%
\url{https://doi.org/10.1109/tkde.2018.2878698}
\showDOI{\tempurl}


\bibitem[\protect\citeauthoryear{Papernot, Abadi, Úlfar Erlingsson,
  Goodfellow, and Talwar}{Papernot et~al\mbox{.}}{2016}]%
        {pate}
\bibfield{author}{\bibinfo{person}{Nicolas Papernot}, \bibinfo{person}{Martín
  Abadi}, \bibinfo{person}{Úlfar Erlingsson}, \bibinfo{person}{Ian
  Goodfellow}, {and} \bibinfo{person}{Kunal Talwar}.}
  \bibinfo{year}{2016}\natexlab{}.
\newblock \showarticletitle{Semi-supervised Knowledge Transfer for Deep
  Learning from Private Training Data}.
\newblock \bibinfo{journal}{\emph{arXiv preprint arXiv:1610.05755}}
  (\bibinfo{year}{2016}).
\newblock
\showeprint[arxiv]{1610.05755}~[stat.ML]


\bibitem[\protect\citeauthoryear{Papernot, McDaniel, Goodfellow, Jha, Celik,
  and Swami}{Papernot et~al\mbox{.}}{2017}]%
        {papernot2017practical}
\bibfield{author}{\bibinfo{person}{Nicolas Papernot}, \bibinfo{person}{Patrick
  McDaniel}, \bibinfo{person}{Ian Goodfellow}, \bibinfo{person}{Somesh Jha},
  \bibinfo{person}{Z~Berkay Celik}, {and} \bibinfo{person}{Ananthram Swami}.}
  \bibinfo{year}{2017}\natexlab{}.
\newblock \showarticletitle{Practical black-box attacks against machine
  learning}. In \bibinfo{booktitle}{\emph{ACM on Asia conference on computer
  and communications security (AsiaCCS)}}.
\newblock


\bibitem[\protect\citeauthoryear{Papernot, Song, Mironov, Raghunathan, Talwar,
  and Erlingsson}{Papernot et~al\mbox{.}}{2018}]%
        {Papernot2018ScalablePL}
\bibfield{author}{\bibinfo{person}{Nicolas Papernot}, \bibinfo{person}{Shuang
  Song}, \bibinfo{person}{Ilya Mironov}, \bibinfo{person}{Ananth Raghunathan},
  \bibinfo{person}{Kunal Talwar}, {and} \bibinfo{person}{{\'U}lfar
  Erlingsson}.} \bibinfo{year}{2018}\natexlab{}.
\newblock \showarticletitle{Scalable Private Learning with PATE}.
\newblock \bibinfo{journal}{\emph{arXiv preprint arXiv:1802.08908}}
  (\bibinfo{year}{2018}).
\newblock


\bibitem[\protect\citeauthoryear{Poirot, Vepakomma, Chang, Kalpathy-Cramer,
  Gupta, and Raskar}{Poirot et~al\mbox{.}}{2019}]%
        {Poirot2019SplitLF}
\bibfield{author}{\bibinfo{person}{Maarten~G. Poirot},
  \bibinfo{person}{Praneeth Vepakomma}, \bibinfo{person}{K. Chang},
  \bibinfo{person}{J. Kalpathy-Cramer}, \bibinfo{person}{R. Gupta}, {and}
  \bibinfo{person}{R. Raskar}.} \bibinfo{year}{2019}\natexlab{}.
\newblock \showarticletitle{Split Learning for collaborative deep learning in
  healthcare}.
\newblock \bibinfo{journal}{\emph{ArXiv}}  \bibinfo{volume}{abs/1912.12115}
  (\bibinfo{year}{2019}).
\newblock


\bibitem[\protect\citeauthoryear{Primault, Lampos, Cox, and
  De~Cristofaro}{Primault et~al\mbox{.}}{2019}]%
        {2019www}
\bibfield{author}{\bibinfo{person}{Vincent Primault},
  \bibinfo{person}{Vasileios Lampos}, \bibinfo{person}{Ingemar Cox}, {and}
  \bibinfo{person}{Emiliano De~Cristofaro}.} \bibinfo{year}{2019}\natexlab{}.
\newblock \showarticletitle{Privacy-Preserving Crowd-Sourcing of Web Searches
  with Private Data Donor}. In \bibinfo{booktitle}{\emph{The World Wide Web
  Conference}} (San Francisco, CA, USA) \emph{(\bibinfo{series}{WWW '19})}.
  \bibinfo{publisher}{Association for Computing Machinery},
  \bibinfo{address}{New York, NY, USA}, \bibinfo{pages}{1487–1497}.
\newblock
\showISBNx{9781450366748}
\urldef\tempurl%
\url{https://doi.org/10.1145/3308558.3313474}
\showDOI{\tempurl}


\bibitem[\protect\citeauthoryear{Ramaswamy, Thakkar, Mathews, Andrew, McMahan,
  and Beaufays}{Ramaswamy et~al\mbox{.}}{2020}]%
        {ramaswamy2020training}
\bibfield{author}{\bibinfo{person}{Swaroop Ramaswamy}, \bibinfo{person}{Om
  Thakkar}, \bibinfo{person}{Rajiv Mathews}, \bibinfo{person}{Galen Andrew},
  \bibinfo{person}{H~Brendan McMahan}, {and} \bibinfo{person}{Fran{\c{c}}oise
  Beaufays}.} \bibinfo{year}{2020}\natexlab{}.
\newblock \showarticletitle{Training production language models without
  memorizing user data}.
\newblock \bibinfo{journal}{\emph{arXiv preprint arXiv:2009.10031}}
  (\bibinfo{year}{2020}).
\newblock


\bibitem[\protect\citeauthoryear{Rana and Weinman}{Rana and Weinman}{2015}]%
        {rana2015data}
\bibfield{author}{\bibinfo{person}{Omer Rana} {and} \bibinfo{person}{Joe
  Weinman}.} \bibinfo{year}{2015}\natexlab{}.
\newblock \showarticletitle{Data as a Currency and Cloud-Based Data Lockers}.
\newblock \bibinfo{journal}{\emph{IEEE Cloud Computing}} \bibinfo{volume}{2},
  \bibinfo{number}{2} (\bibinfo{year}{2015}), \bibinfo{pages}{16--20}.
\newblock


\bibitem[\protect\citeauthoryear{Reddi, Cheng, Kanter, Mattson, Schmuelling,
  Wu, Anderson, Breughe, Charlebois, Chou, et~al\mbox{.}}{Reddi
  et~al\mbox{.}}{2020}]%
        {mlperf:isca20}
\bibfield{author}{\bibinfo{person}{Vijay~Janapa Reddi},
  \bibinfo{person}{Christine Cheng}, \bibinfo{person}{David Kanter},
  \bibinfo{person}{Peter Mattson}, \bibinfo{person}{Guenther Schmuelling},
  \bibinfo{person}{Carole-Jean Wu}, \bibinfo{person}{Brian Anderson},
  \bibinfo{person}{Maximilien Breughe}, \bibinfo{person}{Mark Charlebois},
  \bibinfo{person}{William Chou}, {et~al\mbox{.}}}
  \bibinfo{year}{2020}\natexlab{}.
\newblock \showarticletitle{{MLPerf} Inference Benchmark}. In
  \bibinfo{booktitle}{\emph{International Symposium on Computer Architecture
  (ISCA)}}.
\newblock


\bibitem[\protect\citeauthoryear{Ryffel, Pointcheval, and Bach}{Ryffel
  et~al\mbox{.}}{2020}]%
        {ryffel2020ariann}
\bibfield{author}{\bibinfo{person}{Th{\'e}o Ryffel}, \bibinfo{person}{David
  Pointcheval}, {and} \bibinfo{person}{Francis Bach}.}
  \bibinfo{year}{2020}\natexlab{}.
\newblock \showarticletitle{Ariann: Low-interaction privacy-preserving deep
  learning via function secret sharing}.
\newblock \bibinfo{journal}{\emph{arXiv preprint arXiv:2006.04593}}
  (\bibinfo{year}{2020}).
\newblock


\bibitem[\protect\citeauthoryear{Ryffel, Pointcheval, Bach, Dufour-Sans, and
  Gay}{Ryffel et~al\mbox{.}}{2019}]%
        {ryffel2019partially}
\bibfield{author}{\bibinfo{person}{Th{\'e}o Ryffel}, \bibinfo{person}{David
  Pointcheval}, \bibinfo{person}{Francis Bach}, \bibinfo{person}{Edouard
  Dufour-Sans}, {and} \bibinfo{person}{Romain Gay}.}
  \bibinfo{year}{2019}\natexlab{}.
\newblock \showarticletitle{Partially encrypted deep learning using functional
  encryption}.
\newblock \bibinfo{journal}{\emph{Advances in Neural Information Processing
  Systems}}  \bibinfo{volume}{32} (\bibinfo{year}{2019}),
  \bibinfo{pages}{4517--4528}.
\newblock


\bibitem[\protect\citeauthoryear{Sajadmanesh and Gatica-Perez}{Sajadmanesh and
  Gatica-Perez}{2020}]%
        {sajadmanesh2020differential}
\bibfield{author}{\bibinfo{person}{Sina Sajadmanesh} {and}
  \bibinfo{person}{Daniel Gatica-Perez}.} \bibinfo{year}{2020}\natexlab{}.
\newblock \showarticletitle{When Differential Privacy Meets Graph Neural
  Networks}.
\newblock \bibinfo{journal}{\emph{arXiv preprint arXiv:2006.05535}}
  (\bibinfo{year}{2020}).
\newblock


\bibitem[\protect\citeauthoryear{Sajadmanesh, Jafarzadeh, Ossia, Rabiee,
  Haddadi, Mejova, Musolesi, Cristofaro, and Stringhini}{Sajadmanesh
  et~al\mbox{.}}{2017}]%
        {sajadmanesh2017kissing}
\bibfield{author}{\bibinfo{person}{Sina Sajadmanesh}, \bibinfo{person}{Sina
  Jafarzadeh}, \bibinfo{person}{Seyed~Ali Ossia}, \bibinfo{person}{Hamid~R
  Rabiee}, \bibinfo{person}{Hamed Haddadi}, \bibinfo{person}{Yelena Mejova},
  \bibinfo{person}{Mirco Musolesi}, \bibinfo{person}{Emiliano~De Cristofaro},
  {and} \bibinfo{person}{Gianluca Stringhini}.}
  \bibinfo{year}{2017}\natexlab{}.
\newblock \showarticletitle{Kissing cuisines: Exploring worldwide culinary
  habits on the web}. In \bibinfo{booktitle}{\emph{Proceedings of the 26th
  international conference on world wide web companion}}.
  \bibinfo{pages}{1013--1021}.
\newblock


\bibitem[\protect\citeauthoryear{Shokri and Shmatikov}{Shokri and
  Shmatikov}{2015}]%
        {shokriDNN}
\bibfield{author}{\bibinfo{person}{Reza Shokri} {and} \bibinfo{person}{Vitaly
  Shmatikov}.} \bibinfo{year}{2015}\natexlab{}.
\newblock \showarticletitle{Privacy-Preserving Deep Learning}. In
  \bibinfo{booktitle}{\emph{ACM Conference on Computer and Communications
  Security (CCS)}}.
\newblock


\bibitem[\protect\citeauthoryear{{Shokri}, {Stronati}, {Song}, and
  {Shmatikov}}{{Shokri} et~al\mbox{.}}{2017}]%
        {attack1}
\bibfield{author}{\bibinfo{person}{R. {Shokri}}, \bibinfo{person}{M.
  {Stronati}}, \bibinfo{person}{C. {Song}}, {and} \bibinfo{person}{V.
  {Shmatikov}}.} \bibinfo{year}{2017}\natexlab{}.
\newblock \showarticletitle{Membership Inference Attacks Against Machine
  Learning Models}. In \bibinfo{booktitle}{\emph{IEEE Symposium on Security and
  Privacy (S\&P)}}.
\newblock


\bibitem[\protect\citeauthoryear{Simonyan and Zisserman}{Simonyan and
  Zisserman}{2014}]%
        {vgg}
\bibfield{author}{\bibinfo{person}{Karen Simonyan} {and}
  \bibinfo{person}{Andrew Zisserman}.} \bibinfo{year}{2014}\natexlab{}.
\newblock \showarticletitle{Very Deep Convolutional Networks for Large-Scale
  Image Recognition}.
\newblock \bibinfo{journal}{\emph{arXiv preprint arXiv:1409.1556}}
  (\bibinfo{year}{2014}).
\newblock


\bibitem[\protect\citeauthoryear{Singh, Vepakomma, Gupta, and Raskar}{Singh
  et~al\mbox{.}}{2019}]%
        {Singh2019DetailedCO}
\bibfield{author}{\bibinfo{person}{Abhishek Singh}, \bibinfo{person}{Praneeth
  Vepakomma}, \bibinfo{person}{Otkrist Gupta}, {and} \bibinfo{person}{R.
  Raskar}.} \bibinfo{year}{2019}\natexlab{}.
\newblock \showarticletitle{Detailed comparison of communication efficiency of
  split learning and federated learning}.
\newblock \bibinfo{journal}{\emph{ArXiv}}  \bibinfo{volume}{abs/1909.09145}
  (\bibinfo{year}{2019}).
\newblock


\bibitem[\protect\citeauthoryear{sup Lim, Srivatsa, Chakraborty, and
  Taylor}{sup Lim et~al\mbox{.}}{2018}]%
        {Lim2018LearningLE}
\bibfield{author}{\bibinfo{person}{Yeon sup Lim}, \bibinfo{person}{M.
  Srivatsa}, \bibinfo{person}{S. Chakraborty}, {and} \bibinfo{person}{I.
  Taylor}.} \bibinfo{year}{2018}\natexlab{}.
\newblock \showarticletitle{Learning Light-Weight Edge-Deployable Privacy
  Models}.
\newblock \bibinfo{journal}{\emph{2018 IEEE International Conference on Big
  Data (Big Data)}} (\bibinfo{year}{2018}), \bibinfo{pages}{1290--1295}.
\newblock


\bibitem[\protect\citeauthoryear{Szab{\'o}}{Szab{\'o}}{2014}]%
        {itetoolbox}
\bibfield{author}{\bibinfo{person}{Zolt{\'a}n Szab{\'o}}.}
  \bibinfo{year}{2014}\natexlab{}.
\newblock \showarticletitle{Information Theoretical Estimators Toolbox}.
\newblock \bibinfo{journal}{\emph{Journal of Machine Learning Research}}
  \bibinfo{volume}{15} (\bibinfo{year}{2014}), \bibinfo{pages}{283--287}.
\newblock


\bibitem[\protect\citeauthoryear{{Taram}, {Venkat}, and {Tullsen}}{{Taram}
  et~al\mbox{.}}{2020}]%
        {packetchasing}
\bibfield{author}{\bibinfo{person}{M. {Taram}}, \bibinfo{person}{A. {Venkat}},
  {and} \bibinfo{person}{D. {Tullsen}}.} \bibinfo{year}{2020}\natexlab{}.
\newblock \showarticletitle{Packet Chasing: Spying on Network Packets over a
  Cache Side-Channel}. In \bibinfo{booktitle}{\emph{2020 ACM/IEEE 47th Annual
  International Symposium on Computer Architecture (ISCA)}}.
  \bibinfo{pages}{721--734}.
\newblock
\urldef\tempurl%
\url{https://doi.org/10.1109/ISCA45697.2020.00065}
\showDOI{\tempurl}


\bibitem[\protect\citeauthoryear{Thompson and Warzel}{Thompson and
  Warzel}{2019}]%
        {nytimes2}
\bibfield{author}{\bibinfo{person}{Stuart~A. Thompson} {and}
  \bibinfo{person}{Charlie Warzel}.} \bibinfo{year}{2019}\natexlab{}.
\newblock \bibinfo{title}{The Privacy Project: Twelve Million Phones, One
  Dataset, Zero Privacy}.
\newblock
\newblock
\newblock
\shownote{online--accessed February 2020, url:
  \url{https://www.nytimes.com/interactive/2019/12/19/opinion/location-tracking-cell-phone.html}.}


\bibitem[\protect\citeauthoryear{Tramer and Boneh}{Tramer and Boneh}{2019}]%
        {tramer2018slalom}
\bibfield{author}{\bibinfo{person}{Florian Tramer} {and} \bibinfo{person}{Dan
  Boneh}.} \bibinfo{year}{2019}\natexlab{}.
\newblock \showarticletitle{Slalom: Fast, Verifiable and Private Execution of
  Neural Networks in Trusted Hardware}. In
  \bibinfo{booktitle}{\emph{International Conference on Learning
  Representations (ICLR)}}.
\newblock
\urldef\tempurl%
\url{https://openreview.net/forum?id=rJVorjCcKQ}
\showURL{%
\tempurl}


\bibitem[\protect\citeauthoryear{Van~Goethem, Miramirkhani, Joosen, and
  Nikiforakis}{Van~Goethem et~al\mbox{.}}{2019}]%
        {2019ccsasia}
\bibfield{author}{\bibinfo{person}{Tom Van~Goethem}, \bibinfo{person}{Najmeh
  Miramirkhani}, \bibinfo{person}{Wouter Joosen}, {and} \bibinfo{person}{Nick
  Nikiforakis}.} \bibinfo{year}{2019}\natexlab{}.
\newblock \showarticletitle{Purchased Fame: Exploring the Ecosystem of Private
  Blog Networks}. In \bibinfo{booktitle}{\emph{Proceedings of the 2019 ACM Asia
  Conference on Computer and Communications Security}} (Auckland, New Zealand)
  \emph{(\bibinfo{series}{Asia CCS '19})}. \bibinfo{publisher}{Association for
  Computing Machinery}, \bibinfo{address}{New York, NY, USA},
  \bibinfo{pages}{366–378}.
\newblock
\showISBNx{9781450367523}
\urldef\tempurl%
\url{https://doi.org/10.1145/3321705.3329830}
\showDOI{\tempurl}


\bibitem[\protect\citeauthoryear{Vepakomma, Singh, Gupta, and Raskar}{Vepakomma
  et~al\mbox{.}}{2020}]%
        {Vepakomma2020NoPeekIL}
\bibfield{author}{\bibinfo{person}{Praneeth Vepakomma},
  \bibinfo{person}{Abhishek Singh}, \bibinfo{person}{Otkrist Gupta}, {and}
  \bibinfo{person}{Ramesh Raskar}.} \bibinfo{year}{2020}\natexlab{}.
\newblock \showarticletitle{NoPeek: Information leakage reduction to share
  activations in distributed deep learning}.
\newblock \bibinfo{journal}{\emph{ArXiv}}  \bibinfo{volume}{abs/2008.09161}
  (\bibinfo{year}{2020}).
\newblock


\bibitem[\protect\citeauthoryear{Wagh, Tople, Benhamouda, Kushilevitz, Mittal,
  and Rabin}{Wagh et~al\mbox{.}}{2020}]%
        {falcon}
\bibfield{author}{\bibinfo{person}{Sameer Wagh}, \bibinfo{person}{Shruti
  Tople}, \bibinfo{person}{Fabrice Benhamouda}, \bibinfo{person}{Eyal
  Kushilevitz}, \bibinfo{person}{Prateek Mittal}, {and} \bibinfo{person}{Tal
  Rabin}.} \bibinfo{year}{2020}\natexlab{}.
\newblock \showarticletitle{FALCON: Honest-Majority Maliciously Secure
  Framework for Private Deep Learning}.
\newblock \bibinfo{journal}{\emph{arXiv preprint arXiv:2004.02229}}
  (\bibinfo{year}{2020}).
\newblock


\bibitem[\protect\citeauthoryear{Wang, Zhang, Bao, Zhu, Cao, and Yu}{Wang
  et~al\mbox{.}}{2018}]%
        {notjust}
\bibfield{author}{\bibinfo{person}{Ji Wang}, \bibinfo{person}{Jianguo Zhang},
  \bibinfo{person}{Weidong Bao}, \bibinfo{person}{Xiaomin Zhu},
  \bibinfo{person}{Bokai Cao}, {and} \bibinfo{person}{Philip~S. Yu}.}
  \bibinfo{year}{2018}\natexlab{}.
\newblock \showarticletitle{Not Just Privacy}.
\newblock \bibinfo{journal}{\emph{ACM International Conference on Knowledge
  Discovery and Data Mining (KDD)}} (\bibinfo{year}{2018}).
\newblock
\showISBNx{9781450355520}
\urldef\tempurl%
\url{https://doi.org/10.1145/3219819.3220106}
\showDOI{\tempurl}


\bibitem[\protect\citeauthoryear{Zhang, Song, and Qi}{Zhang
  et~al\mbox{.}}{2017}]%
        {utkface}
\bibfield{author}{\bibinfo{person}{Zhifei Zhang}, \bibinfo{person}{Yang Song},
  {and} \bibinfo{person}{Hairong Qi}.} \bibinfo{year}{2017}\natexlab{}.
\newblock \showarticletitle{Age Progression/Regression by Conditional
  Adversarial Autoencoder}.
\newblock \bibinfo{journal}{\emph{2017 IEEE Conference on Computer Vision and
  Pattern Recognition (CVPR)}} (\bibinfo{year}{2017}),
  \bibinfo{pages}{4352--4360}.
\newblock


\bibitem[\protect\citeauthoryear{Zheleva and Getoor}{Zheleva and
  Getoor}{2009}]%
        {2009www}
\bibfield{author}{\bibinfo{person}{Elena Zheleva} {and} \bibinfo{person}{Lise
  Getoor}.} \bibinfo{year}{2009}\natexlab{}.
\newblock \showarticletitle{To Join or Not to Join: The Illusion of Privacy in
  Social Networks with Mixed Public and Private User Profiles}. In
  \bibinfo{booktitle}{\emph{Proceedings of the 18th International Conference on
  World Wide Web}} (Madrid, Spain) \emph{(\bibinfo{series}{WWW '09})}.
  \bibinfo{publisher}{Association for Computing Machinery},
  \bibinfo{address}{New York, NY, USA}, \bibinfo{pages}{531–540}.
\newblock
\showISBNx{9781605584874}
\urldef\tempurl%
\url{https://doi.org/10.1145/1526709.1526781}
\showDOI{\tempurl}


\end{thebibliography}

\appendix



\section{Appendix}

\subsection{Theorem for Upper bound on $I(\mathbf{x_c};\mathbf{u})$}
\begin{thm} \label{thm:gaussian-max}
Given a random vector $\mathbf{x}\in R^n$ with covariance matrix $\mathbf{K}$, then:
\begin{equation}\label{eq:upper:2}
    \mathcal{H}(\mathbf{x})\leq \frac{1}{2}\log((2\pi e)^n|\mathbf{K}|)
\end{equation}
\end{thm}
\begin{proof}
This theorem is proved using the fact that the KL-divergence of two distributions is always positive. The complete proof is in~\cite{elements-book}, Theorem 8.6.5. 
\end{proof}
\subsection{Lower bound on $I(\mathbf{x_c};\mathbf{c})$}

First, we introduce a lemma~\cite{osia1} that we use for finding the lower bound of $I(\mathbf{x_c}; \mathbf{c})$. 

\begin{lemma}
For any arbitrary conditional distribution $q(\mathbf{c}|\mathbf{x_c})$, we have:
\begin{equation}\label{eq:lemma:final}
     \mathbb{E}_{\mathbf{x_c}, \mathbf{c}}[\log\frac{q(\mathbf{c}|\mathbf{x_c})}{p(\mathbf{c})}]
     \leq I(\mathbf{x_c};\mathbf{c})
\end{equation}
\end{lemma}

\begin{proof}

Since we know that KL-divergence is always non-negative, we can write: 
\begin{equation*}
    D_{KL}(p(\mathbf{c}|\mathbf{x_c})||q(\mathbf{c}|\mathbf{x_c}))=  \int p(\mathbf{c}|\mathbf{x_c})\;\log{\frac{p(\mathbf{c}|\mathbf{x_c})}{q(\mathbf{c}|\mathbf{x_c})}}\; d\mathbf{c} \geq 0
\end{equation*}

From this, we can come to:
\begin{equation*} 
   \int p(\mathbf{c},\mathbf{x_c})\;\log{\frac{p(\mathbf{c}|\mathbf{x_c})p(\mathbf{c})}{q(\mathbf{c}|\mathbf{x_c})p(\mathbf{c})}}\;d\mathbf{c}\;d\mathbf{x_c} \geq 0  
\end{equation*}

By negation, we get:
\begin{equation}\label{eq:lemma1:pt1}
       -\int p(\mathbf{c},\mathbf{x_c})\;\log{\frac{p(\mathbf{c}|\mathbf{x_c})p(\mathbf{c})}{q(\mathbf{c}|\mathbf{x_c})p(\mathbf{c})}}\;d\mathbf{c}\;d\mathbf{x_c} \leq 0 
\end{equation}

On the other hand, from the definition of mutual information, we can write:

\begin{equation}\label{eq:lemma1:pt2}
    I(\mathbf{x_c};\mathbf{c}) = \int p(\mathbf{c},\mathbf{x_c})\;\log{\frac{p(\mathbf{c},\mathbf{x_c})}{p(\mathbf{c})p(\mathbf{x_c})}}\;d\mathbf{x_c}\;d\mathbf{c} 
\end{equation}

If we add  $I(\mathbf{x_c};\mathbf{c})$ from Equation~\ref{eq:lemma1:pt2} to ~\ref{eq:lemma1:pt1}, we get:

\begin{equation*}
    \int p(\mathbf{x_c}, \mathbf{c})\;\log\frac{q(\mathbf{c}|\mathbf{x_c})}{p(\mathbf{c})} \leq I(\mathbf{x_c};\mathbf{c})
\end{equation*}

Which yields:
\begin{equation}
     \mathbb{E}_{\mathbf{x_c}, \mathbf{c}}[\log\frac{q(\mathbf{c}|\mathbf{x_c})}{p(\mathbf{c})}]
     \leq I(\mathbf{x_c};\mathbf{c})
\end{equation}
\end{proof}


Now, we review the theorem and prove it. 

\textbf{Theorem 3.2.}\textit{
The lower bound on} $I(\mathbf{x_c};\mathbf{c})$ is:
\begin{equation}  
    \mathcal{H}(\mathbf{c}) + \max\limits_{q}\  \mathbb{E}_{\mathbf{x_c}, \mathbf{c}}[\log\; q(\mathbf{c}|\mathbf{x_c})]
\end{equation}
\textit{By $q$, we mean all members of a possible family of distributions for this conditional probability.}

\begin{proof}
For all $q$, the left hand side of equation~\ref{eq:lemma:final} offers a lower bound. The equality happens when $q(\mathbf{c}|\mathbf{x_c})$ is equal to $p(\mathbf{c}|\mathbf{x_c})$. Given this, if we estimate a close enough distribution $q$ that maximizes the left hand side of the inequality~\ref{eq:lemma:final}, we can find a tight lower bound for the mutual information. We can re-write inequality~\ref{eq:lemma:final} as:
\begin{equation*}
      - \mathbb{E}_{\mathbf{c}}[\log\ p(\mathbf{c})] +\mathbb{E}_{\mathbf{x_c}, \mathbf{c}}[\log\ q(\mathbf{c}|\mathbf{x_c})]
     \leq I(\mathbf{x_c};\mathbf{c})
\end{equation*}

Based on the definition of Entropy and the discussion above about tightening the bound, the lower bound on the mutual information is:
\begin{equation}
    \mathcal{H}(\mathbf{c}) + \max\limits_{q}\  \mathbb{E}_{\mathbf{x_c}, \mathbf{c}}[\log\; q(\mathbf{c}|\mathbf{x_c})]
\end{equation}

\end{proof}

%
%
%
%
%
%
%
%
%
%
%
%



\subsection{Hyperparameters for Training} \label{sec:hyper}

Tables~\ref{tab:hp1}, ~\ref{tab:hp2} and ~\ref{tab:hp3} show the hyperparameters used for training in the experiments of Sections~\ref{sec:res2},~\ref{sec:adv} and~\ref{sec:black-box}. 
For the first one, the $Point\#$ indicates the experiment that produced the given point in the graph, if the points were numbered from left to right. 
The hyperparameters of the rest of the experiments are the same as the ones brought. 
In our implementation, for ease of use and without loss of generality, we have introduced a variable $\gamma$ to the loss function in Equation~\ref{eq:loss2}, in a way that $\gamma=\frac{1}{\lambda}$. With this introduction, we do not directly assign a $\lambda$ (as if $\lambda$ were removed and replaced by $\gamma$ as a coefficient of the other term). 
In the tables, we have used lambda to be consistent, and in the cells where the value for $\lambda$ is not given, it means that the loss is only cross-entropy. But in the Code, the coefficient is set on the other term and is $1/\lambda$s reported here.  The batch sizes used for training are 128 for CIFAR-100, MNIST, and UTKFace and 40 and 30 for CelebA. For testing the batch size is 1, so as to sample a new noise tensor for each image and capture the stochasticity. Also, the number of samples taken for each update in optimization is 1 since we do mini-batch training and for each mini-batch we take a new sample. Finally, $M$ is set to $1.5$ for all benchmarks, except for CelebA where it is set to be $5$. 

\begin{table*}[t]
    \centering
    \caption{hyper parameters for Section~\ref{sec:res2}}\label{tab:hp1}
    \resizebox{0.7\textwidth}{!}{
    \setlength{\tabcolsep}{1pt}
\begin{tabular*}{\textwidth}{@{\extracolsep{\fill}}lllllllllll@{}}
\toprule
\multirow{2}{*}{Model}     & \multirow{2}{*}{Point\#} & \multicolumn{3}{c}{Training Phase 1} & \multicolumn{3}{c}{Training Phase 2} \\\cmidrule{3-5} \cmidrule{6-8}
                           &                          & epoch     & LR        & $\lambda$    & epoch     & LR        & $\lambda$    \\\midrule
\multirow{5}{*}{CIFAR-100} & 1                        & 17        & 0.001     & 1            & 3         & 0.001     & 10           \\
                           & 2                        & 24        & 0.001     & 1            & 2         & 0.001     & 10           \\
                           & 3                        & 30        & 0.001     & 1            & 2         & 0.001     & 10           \\
                           & 4                        & 40        & 0.001     & 0.2          & 2         & 0.001     & 10           \\
                           & 5                        & 140       & 0.001     & 0.2          & 2         & 0.001     & 10           \\\midrule
\multirow{5}{*}{MNIST}     & 1                        & 50        & 0.01      & 100          & 90        & 0.001     & 200          \\
                           & 2                        & 50        & 0.01      & 100          & 160       & 0.001     & 200          \\
                           & 3                        & 50        & 0.01      & 100          & 180       & 0.001     & 200          \\
                           & 4                        & 50        & 0.01      & 100          & 260       & 0.001     & 100          \\
                           & 5                        & 50        & 0.01      & 100          & 290       & 0.001     & 100          \\\midrule
\multirow{5}{*}{UTKFace}   & 1                        & 6         & 0.01      & 0.1          & 6         & 0.0001    & 100          \\
                           & 2                        & 4         & 0.01      & 0.1          & 2         & 0.0001    & 100          \\
                           & 3                        & 8         & 0.01      & 0.1          & 2         & 0.0001    & 100          \\
                           & 4                        & 10        & 0.01      & 0.1          & 2         & 0.0001    & 100          \\
                           & 5                        & 12        & 0.01      & 0.1          & 2         & 0.0001    & 100       \\\bottomrule  
\end{tabular*}

    }
\end{table*}

\begin{table*}[t]
    \centering
    \caption{hyper parameters for Section~\ref{sec:black-box}}\label{tab:hp2}
    \resizebox{0.7\textwidth}{!}{
    \setlength{\tabcolsep}{1pt}
\begin{tabular*}{\textwidth}{@{\extracolsep{\fill}}lllllllllll@{}}
\toprule
\multirow{2}{*}{Model}    & \multirow{2}{*}{Point\#} & \multicolumn{3}{c}{Training Phase 1} & \multicolumn{3}{c}{Training Phase 2} & \multicolumn{3}{c}{Training Phase 3} \\\cmidrule{3-5} \cmidrule{6-8} \cmidrule{9-11}
                          &                          & epoch     & LR       & $\lambda$     & epoch     & LR        & $\lambda$    & epoch     & LR        & $\lambda$    \\\midrule
\multirow{5}{*}{VGG16}    & 1                        & 0.5       & 0.01     & 1             & 0.5       & 0.001     & 1            & -         & -         & -            \\
                          & 2                        & 0.5       & 0.01     & 1             & 0.7       & 0.001     & 1            & -         & -         & -            \\
                          & 3                        & 0.5       & 0.01     & 1             & 0.8       & 0.001     & 1            & -         & -         & -            \\
                          & 4                        & 0.8       & 0.01     & 1             & 0.8       & 0.001     & 1            & 0.2       & 0.001     & 5            \\
                          & 5                        & 1         & 0.01     & 1             & 0.8       & 0.001     & 1            & 0.2       & 0.001     & 100          \\\midrule
\multirow{5}{*}{ResNet18} & 1                        & 1         & 0.01     & 10            & 0.5       & 0.001     & 1            & -         & -         & -            \\
                          & 2                        & 1         & 0.01     & 5             & 0.5       & 0.001     & 1            & -         & -         & -            \\
                          & 3                        & 1         & 0.01     & 5             & 0.7       & 0.001     & 1            & -         & -         & -            \\
                          & 4                        & 1.2       & 0.01     & 3             & 0.5       & 0.001     & 1            & 0.2       & 0.001     & 5            \\
                          & 5                        & 2         & 0.01     & 5             & 0.5       & 0.001     & 1            & 0.2       & 0.001     & 5  \\\bottomrule         
\end{tabular*}
    }
\end{table*}

\begin{table*}[t]
    \centering
    \caption{hyper parameters for Section~\ref{sec:adv}}\label{tab:hp3}
    \resizebox{0.7\textwidth}{!}{
    \setlength{\tabcolsep}{0pt}
\begin{tabular*}{\textwidth}{@{\extracolsep{\fill}}llrrrrrr@{}}
\toprule
\multirow{2}{*}{Model}     & \multirow{2}{*}{SR(\%)} & \multicolumn{2}{c}{Training Phase 1} & \multicolumn{2}{c}{Training Phase 2} & \multicolumn{2}{c}{Training Phase 3} \\\cmidrule{3-4}\cmidrule{5-6}\cmidrule{7-8}
                                   &                     & epoch             & LR               & epoch            & LR                & epoch            & LR                \\\midrule
\multirow{7}{*}{Adversary-hair}    & 00.00                & 1                 & 0.01             & -                & -                 & -                & -                 \\
                                   & 33.60               & 1                 & 0.01             & 2                & 0.0001            & 1                & 0.00001           \\
                                   & 53.70               & 1                 & 0.01             & 2                & 0.0001            & 1                & 0.00001           \\
                                   & 71.00               & 1                 & 0.01             & 2                & 0.0001            & 1                & 0.00001           \\
                                   & 89.70               & 1                 & 0.01             & 2                & 0.0001            & 3                & 0.00001           \\
                                   & 95.60               & 1                 & 0.01             & 2                & 0.0001            & 2                & 0.00001           \\
                                   & 98.30               & 1                 & 0.01             & 2                & 0.0001            & 3                & 0.00001           \\\midrule
\multirow{7}{*}{Adversary-glasses} & 00.00                & 1                 & 0.01             & -                & -                 & -                & -                 \\
                                   & 33.60               & 1                 & 0.01             & 2                & 0.0001            & 1                & 0.00001           \\
                                   & 53.70               & 1                 & 0.01             & 2                & 0.0001            & 1                & 0.00001           \\
                                   & 71.00               & 1                 & 0.01             & 2                & 0.0001            & 1                & 0.00001           \\
                                   & 89.70               & 1                 & 0.01             & 2                & 0.0001            & 3                & 0.00001           \\
                                   & 95.60               & 1                 & 0.01             & 2                & 0.0001            & 2                & 0.00001           \\
                                   & 98.30               & 1                 & 0.01             & 2                & 0.0001            & 3                & 0.00001        \\\bottomrule  
\end{tabular*}
    }
\end{table*}

\subsection{Code Directory Structure}

The code and model checkpoints used to produce the results are provided at \url{https://github.com/mireshghallah/cloak-www-21}. The code is in the directory \texttt{code}  and the models and NumPy files are named \texttt{saved\_nps.zip} and they both have the same directory structure. They each contain 5 Folders named \textit{exp1-trade-off, exp2-adversary, exp3-black-box, exp4-fairness} and \textit{exp5-shredder} which are related to the results in the experiments section in the same order. The pre-trained parameters needed are provided in the  \texttt{saved\_nps.zip}, in the corresponding directory. So, all that is needed to be done is to copy all files from the \texttt{saved\_nps.zip
} directory to their corresponding positions in the code folders, and then run the provided Jupyter notebooks. 
%
%
The notebooks that were used to generate representations are provided, in case someone wants to reproduce the results, and the saved \sieve models and pre-trained models are given as well. 
For acquiring the datasets, you can have a look at the \texttt{acquire\_datasets.ipynb} notebook, included in the \texttt{code.zip}.

In short, each notebook has \sieve in its name will start by loading the required datasets and then creating a model. Then, the model is trained based on the experiments and using the hyperparameters provided in section~\ref{sec:hyper}.
Finally, you can run a test function that is provided to evaluate the model. 
For seeing how the mutual information is estimated, you can run the notebooks that have \texttt{mutual\_info} in their names.
You need not have run the training beforehand if you place the provided \texttt{.npy} files in the correct directories. For the mutual information estimation, you will need to download the ITE toolbox~\cite{itetoolbox}. The link is provided in the code. 

\subsection{Fairness Metrics}
In a classification task, demographic parity requires the conditional probability of the classifier predicting output class $\hat{Y}= y$ given sensitive variable $S = 0$ to be the same as predicting class $\hat{Y} = y$ given $S=1$. In other words, $P(\hat{Y}= y| S=0) = P(\hat{Y} = y| S=1)$. 
Since in most real cases these values are not the same, the maximum pair-wise difference between these values is considered as a measure of fairness, $\Delta_{DemP}$, and the lower this difference, the more fair the classifier. Here $S$ would be the gender, which due to the data provided in the dataset, is binary. 
We have only two target classes of black hair and non-black hair,  so the $\Delta_{DemP}(y=0)$ is the same as $\Delta_{DemP}(y=1)$. 

Equalized odds is another fairness measure, which requires the conditional probability of the classifier predicting  class $\hat{Y}= y$ given sensitive variable $S=0$ and ground truth class $Y=y$ be equal to the same conditional probability but with $S=1$. In other words, $P(\hat{Y}=y|S=0, Y=y) = P(\hat{Y}=y|S=1, Y=y)$. 
Similar to the demographic parity case, we also measure the difference in these conditional probabilities for both $y=1$ (black hair) and $y=0$ (non-black hair) and report their summation as $\Delta_{EO}$, commensurate with~\cite{madras}.
%




\end{document}